\documentclass{article}

\usepackage{mathrsfs,amsmath,amsfonts,amssymb,amstext,amscd, dsfont,pifont,
            amsthm,euscript,color,xcolor,accents,xr} 
\usepackage{algorithmic}
\usepackage{algorithm}
\usepackage{url}
\usepackage{mathtools}
\usepackage{float}
\usepackage{appendix}
\usepackage{nicefrac}
\usepackage{anysize}
\usepackage{enumerate}
\usepackage{graphicx}
\usepackage{upgreek}
\usepackage{pdfpages}
\usepackage{subcaption}  
\usepackage{bbm,bm}
\usepackage{tikz}

\usepackage[linecolor=magenta!60!, backgroundcolor=magenta!10!,textwidth=1.6cm, textcolor=magenta]{todonotes}

\numberwithin{equation}{section} 

\input{macros_nicole}

\usepackage[onehalfspacing]{setspace} 

\definecolor{citecolor}{rgb}{0, 0.3, 0.6}

\usepackage[round]{natbib}
\usepackage[colorlinks=true,linkcolor=citecolor, citecolor=citecolor, backref]{hyperref}
\newcommand{\mycite}[1]{\cite{#1}} 
\newcommand{\mycitep}[1]{\citep{#1}}

\title{A Smoothed Discrepancy Principle for Random Feature Methods and Neural Networks }

\author{Mike Nguyen\footnote{Corresponding Author} \\
Technical University  of Braunschweig \\
\texttt{mike.nguyen@tu-braunschweig.de} 
\and
Nicole M\"ucke\\
Technical University  of Braunschweig  \\ 
\texttt{nicole.muecke@tu-braunschweig.de}
}
\date{\today}

\begin{document}

\maketitle

\begin{abstract}
We study data–driven early stopping for spectral regularization methods in the classical nonparametric regression setting. Building on the discrepancy principle, we propose a multi-scale stopping rule that applies to general kernel estimators and show that, unlike previous approaches, it achieves full adaptivity over \textit{all} smoothness levels in the well-specified case. A key contribution of our work is a robustified extension based on random feature approximations, which reduces computational cost on large datasets while preserving minimax‐optimal statistical guarantees. Our procedure not only selects an optimal stopping time but also provides a fully data-driven choice of the number of random features needed to achieve optimal rates. Through the established connection between random features and neural networks in the NTK regime, our method further yields a principled, data-driven recommendation for the network width. We prove that the resulting simultaneously chosen width and stopping time allow neural networks to attain minimax-optimal learning rates without prior knowledge of smoothness or capacity parameters. 
\end{abstract}


\section{Introduction}

Overfitting remains a pervasive challenge in machine learning, particularly in regression problems characterized by noisy data, where regularization is indispensable for ensuring robust generalization. In this context, early stopping provides a powerful remedy. It can be interpreted either as an \emph{implicit} regularization technique, where iterative procedures like gradient descent are halted at a strategically chosen time, or as an \emph{explicit} form of regularization, where the stopping rule identifies an optimal smoothing parameter, analogous to the role of the penalty parameter in Tikhonov regularization. 
Driven by the dual requirements of predictive accuracy and computational efficiency, a diverse array of early stopping strategies has emerged, each navigating the trade-off between rigorous theoretical guarantees and empirical flexibility.

\paragraph{Cross-validation and hold-out methods.} The most widely adopted approach is to set aside a validation set and select the model (or stopping time) that minimizes the validation error. Cross-validation (CV) and simple hold-out validation are popular due to their conceptual simplicity and broad applicability \citep{Stone74,Arlot2010}.  Nevertheless, the convenience of CV incurs a significant computational overhead: the model must be retrained multiple times. In extreme cases such as leave-one-out CV (LOO), the model is retrained $n$ times for a dataset of size $n$, which is often prohibitively expensive \citep{Craven1979}. Furthermore, utilizing a hold-out set reduces the effective data available for training, and CV estimates can suffer from high variance when the sample size is small. In summary, while CV is a versatile default, it is often \emph{data-hungry} and \emph{time-consuming}, failing to directly leverage theoretical insights into the underlying learning problem.

\paragraph{Lepskii's method and balancing schemes.} An alternative line of research avoids explicit validation splits by employing \emph{data-dependent parameter tuning} based on the principle of balancing bias and variance. A prominent example is \emph{Lepskii's method}, originally developed in the field of non-parametric function estimation, which selects the highest model complexity that does not significantly deviate from simpler estimates. Intuitively, one computes solutions across multiple regularization levels and identifies the weakest regularization that maintains statistical stability \citep{doi:10.1137/1135065, Math2003GeometryOL, DeVito2010}. While these methods are known to be \emph{minimax adaptive} in various settings, this adaptivity typically comes at the cost of a mild penalty in the convergence rate, often a logarithmic factor. Moreover, balancing schemes share a similar \emph{computational burden} with CV, as they require solving the learning problem for a dense range of regularization parameters.

\paragraph{Discrepancy principle.} Rooted in regularization theory for 
inverse problems, the \emph{discrepancy principle} 
\citep{Morozov1966OnTS, cb70d984112e46269082651d324326e7,articlemathe,stankewitz2020smoothedresidualstoppingstatistical} provides a conceptually elegant stopping rule 
based on the magnitude of the residuals. Subsequent theoretical 
advancements have successfully adapted this principle to kernel-based 
learning in the classical regression framework \citep{blanchard2017optimaladaptationearlystopping, 
celisse:hal-02548917}. The core mechanism involves monitoring the training 
trajectory and halting the iteration as soon as the empirical residual norm 
reaches a threshold proportional to the expected noise level of the data. 
This principle serves as the primary inspiration for 
our work. We provide a technical exposition of the 
discrepancy principle and its limitations in Section~2.

\paragraph{Other model selection techniques.} Beyond validation and balancing, several approaches draw from statistical learning theory to derive stopping criteria based on complexity bounds. For instance, Rademacher complexity provides data-dependent generalization bounds that can be monitored to halt training when the empirical risk reduction no longer justifies the increase in model complexity \citep{raskutti2013earlystoppingnonparametricregression, wei2018earlystoppingkernelboosting}. Similarly, PAC-Bayesian frameworks offer a principled way to derive stopping rules by optimizing an upper bound on the risk, leveraging a prior distribution over the hypothesis space to penalize overly complex models \citep{ishibashi2020stoppingcriterionactivelearning}. 
In contrast to these bound-based methods, heuristic approaches have also been proposed. For example, the Label Wave method by \cite{yuan2025earlystoppinglabelnoise}, which identifies stopping points based on successive changes in training predictions. Despite extensive empirical validation, the Label Wave method is not yet supported by theoretical guarantees.
For a comprehensive review and empirical comparison of early stopping methods, we refer the reader to \cite{ziebell2025earlystoppingimplicitregularizationiterative}, which provides extensive simulations alongside corresponding Python implementations.

\paragraph{Our approach.} In this paper, we propose a novel early stopping rule for kernel-based learning algorithms that builds upon the smoothed discrepancy principle of \cite{celisse:hal-02548917} while overcoming its primary limitations. Our method does \emph{not} require a validation set (unlike CV). It automatically calibrates a residual threshold from the data, enabling a fully \emph{adaptive} stopping criterion that achieves minimax-optimal convergence rates. To ensure scalability to large-scale datasets, we incorporate \emph{random feature approximations} \citep{NIPS2007_013a006f,nguyen2023random}, thereby enabling efficient computation. Furthermore, we extend our framework to neural networks via the \emph{Neural Tangent Kernel (NTK)} regime \citep{jacot2018neural,nguyen2023neurons}, providing a theoretically grounded early stopping strategy for over-parameterized models.

Specifically, the main contributions of this work are as follows: 

\begin{itemize}
    \item We introduce a new, fully data-driven early stopping rule for kernel methods, which, to the best of our knowledge, is the first to be integrated with random feature approximations. The procedure adapts to the unknown smoothness of the target function and achieves minimax-optimal convergence rates without oracle knowledge or validation data.
    
    \item We extend this strategy to neural networks in the NTK regime, establishing that training a sufficiently wide network with gradient descent and our stopping rule yields generalization performance equivalent to an optimally regularized kernel method. This bridges non-parametric theory and modern deep learning.
    
    \item Our algorithm identifies not only the stopping time but also the sufficient number of random features (or, equivalently, the required network width) to guarantee optimal learning rates in a data-driven manner.
\end{itemize} 

\paragraph{Outline.} The remainder of this paper is structured as follows. 
In Section~2, we formalize the problem setting, review the necessary 
mathematical foundations of kernel methods and random feature 
approximations, and introduce our proposed adaptive stopping rule. 
Section~3 is dedicated to our main theoretical contributions, where 
we establish the minimax-optimal convergence rates of the proposed 
method. In Section~4, we extend our framework to the training 
dynamics of over-parameterized neural networks via the Neural Tangent 
Kernel (NTK) regime. Finally, Section~5 provides numerical 
experiments on both synthetic and real-world datasets to substantiate 
our theoretical findings. All proofs are deferred to the Appendix.

\medskip

{\bf Notation.} 
By $\cL(\cH_1, \cH_2)$ we denote the space of bounded linear operators between real separable Hilbert spaces $\cH_1$, $\cH_2$. 
We write $\cL(\cH, \cH) = \cL(\cH)$. For $\Gamma \in \cL(\cH)$, we denote by $\Gamma^*$ the adjoint operator. For a compact operator $\Gamma \in \cL(\mathcal{H})$, the trace is defined by $\operatorname{tr}(\Gamma) = \sum_{k=1}^\infty \langle \Gamma e_k, e_k \rangle,$ where $\{e_k\}_{k=1}^\infty$ is any orthonormal basis of $\mathcal{H}$. 
We denote by $\mathcal{F}(\mathcal{X},\mathcal{Y})$ the space of measurable functions from $\mathcal{X}$ to $\mathcal{Y}$.  
We write $L^2(\mathcal{X},\rho_x):= L^2(\mathcal{X},\rho_{x};\mathcal{Y})$ equipped with the norm $\|f\|^2_{L^2(\rho_x)}:=\int_{\mathcal{X}}f(x)^2 \, d\rho_x(x)$.  We let $[n] := \{1,\dots,n\}$ and denote the output vector by $Y = (y_1,\dots,y_n) \in \mathcal{Y}^n \subset \mathbb{R}^n$ with norm $\|Y\|^2_n:=\frac{1}{n}\|Y\|_2^2 = \frac{1}{n}\sum_{i=1}^n y_i^2$.


\section{Mathematical Framework}
\label{sec:setting}

We consider the standard nonparametric regression setting. Let $(x, y) \in \mathcal{X} \times \mathcal{Y} \subset \mathbb{R}^d \times \mathbb{R}$ be a pair of random variables satisfying the regression model
\begin{align}\label{regressionmodel}
    y = g^*(x) + \epsilon,
\end{align}
where $g^* : \mathcal{X} \to \mathbb{R}$ is the unknown target function and $\epsilon$ is a real-valued random variable such that $\mathbb{E}[\epsilon \mid x] = 0$ and $\mathbb{E}[\epsilon^2 \mid x] = \sigma^2$ for some finite variance $\sigma^2 \in (0,\infty)$.

Following the assumptions of \mycite{raskutti2013earlystoppingnonparametricregression} and \mycite{celisse:hal-02548917}, we consider the noise variance $\sigma^2$ to be known. However, it has been shown that $\sigma^2$ can be consistently estimated from data without affecting the theoretical guarantees (\mycite{celisse:hal-02548917}, Remark 26). 

We observe a dataset of $n$ independent and identically distributed (i.i.d.) samples $\{(x_i, y_i)\}_{i=1}^n$ drawn from the joint distribution of $(x,y) \sim \rho$. The marginal distribution of $x$ is denoted by $\rho_x$. We also write $Y := (y_1, \dots, y_n)$ for the vector of responses.

In the next subsection, we briefly recall the kernel-based framework; see also \mycite{steinwart2008support, Muecke2017op.rates, Lin_2020}.

\subsection{Kernel Methods and Regularization: General Setup}

Let $K : \mathcal{X} \times \mathcal{X} \longrightarrow \mathbb{R}$ denote 
a real-valued positive kernel.

We denote by $\mathcal{H}$ the associated unique RKHS, which can be continuously embedded into $L^2(\mathcal{X},\rho_{x};\mathcal{Y})$. 
We assume that $K$ is bounded, implying that the inclusion 
$\mathcal{S} : \mathcal{H} \hookrightarrow L^2(\mathcal{X},\rho_{x};\mathcal{Y})$ 
is a bounded linear map.

\medskip
\textbf{Regularization.}
A standard approach to estimating $g^*$ is empirical risk minimization (ERM) over the hypothesis space $\mathcal{H}$,
\[
\min_{f \in \mathcal{H}} \widehat{\mathcal{E}}(f),
\qquad  
\widehat{\mathcal{E}}(f) = \frac{1}{n} \sum_{i=1}^n (f(x_i) - y_i)^2.
\]
However, direct ERM in an RKHS is typically ill-posed. To prevent overfitting and ensure consistency, one introduces \emph{regularization}.

\begin{definition}[Regularization function]
Let $\phi : [0,1]\times(0,\kappa^2]  \to \mathbb{R}$, and define $\phi_t(x) := \phi(\frac{1}{t}, x)$ for any $t \in \mathbb{N}$.  
The family $\{\phi_t\}_{t \ge 1}$ is called a family of \emph{regularization functions} if $(t,x) \mapsto x \phi_t(x)$ is monotone increasing in  $t$ and there exists a constant $E > 0$ such that for all $t \in \mathbb{N}$:
\begin{align}
\text{(i)} \quad & \sup_{0 < x \le \kappa^2} |x \phi_t(x)| \le 1, \label{def.phi} \\
\text{(ii)} \quad & \sup_{0 < x \le \kappa^2} |\phi_t(x)| \le Et, \label{def.phi2}\\
\text{(iii)} \quad & \sup_{0 < x \le \kappa^2} |r_t(x)| \le 1,
\qquad r_t(x) := 1 - x \phi_t(x). \label{residual}
\end{align}
\end{definition}

This framework encompasses both explicit regularization (e.g., Tikhonov regularization) and implicit regularization through iterative schemes such as gradient descent and its accelerated variants.  
Originally developed in the context of inverse problems \mycitep{cb70d984112e46269082651d324326e7}, such methods have since been widely applied in machine learning, particularly to nonparametric least-squares regression \mycitep{Caponetto, Muecke2017op.rates, Lin_2020}.

It has been shown in \mycite{10.1162/neco.2008.05-07-517, Muecke2017op.rates} that achievable learning rates are determined by the \emph{qualification} of the regularization family $\{\phi_t\}_{t \ge 1}$, i.e., the largest $\nu > 0$ such that for all $q \in [0,\nu]$ and $t\in\mathbb{N}$,
\begin{equation}
\label{c_r}
\sup_{0 < x \le \kappa^2} |r_{t}(x)| \, x^{q} \le c_q \, t^{-q},
\end{equation}
for some constant $c_q > 0$.

\begin{assumption}
\label{LFL}
There exists a constant $c_\phi > 0$ such that
\[
\phi_t(x) \, x \ge c_\phi \,\bigl(1 \wedge x t\bigr)
\qquad \forall\, (x, t) \in [0,\kappa^2] \times \mathbb{N}.
\]
\end{assumption}

This assumption holds, for instance, for Tikhonov regularization, gradient descent, and Showalter’s method with $c_\phi = 1/2$.

\medskip
A principled approach leverages the spectral structure of the empirical operators
$\widehat{\Sigma} := \widehat{\mathcal{S}}^{\,*} \widehat{\mathcal{S}} : \mathcal{H} \to \mathcal{H}$ and $\widehat{\mathcal{S}} : \mathcal{H} \to \mathbb{R}^n$, $\widehat{\mathcal{S}}^{\,*} : \mathbb{R}^n \to \mathcal{H}$, defined by
\begin{align*}
\widehat{\Sigma} f = \frac{1}{n} \sum_{i=1}^n K(x_i, \cdot) \, f(x_i), \quad
\bigl(\widehat{\mathcal{S}} f\bigr)_{j} 
    = \left\langle f, K(x_{j}, \cdot) \right\rangle_{\mathcal{H}}, \quad
\widehat{\mathcal{S}}^{\,*} Y = \frac{1}{n} \sum_{i=1}^n K(x_i, \cdot) \, y_i.
\end{align*}
With these operators in hand, spectral regularization estimators take the form
\begin{equation}
\label{def:estimator}
\widehat{\mathcal{S}} f_t
= \widehat{\mathcal{S}} \, \phi_t(\widehat{\Sigma}) \, \widehat{\mathcal{S}}^{*} Y
= \phi_t(K)\, K\, Y
\;\in\; \mathbb{R}^n,
\end{equation}
where $K \in \mathbb{R}^{n \times n}$ denotes the kernel Gram matrix with entries 
$K_{i,j} = \frac{1}{n}K(x_i, x_j)$.

Given a kernel estimator $f_t \in \mathcal{H}$, our goal is to select a stopping time $t > 0$ that minimizes the excess risk
\[
\mathcal{R}(f_t) := \| g^* - \mathcal{S} f_t \|^2_{L^2(\rho_x)}.
\]

\subsection{Stopping Time Motivation}

\textbf{Classical Discrepancy Principle.}
Our stopping rule builds upon the well-known \emph{discrepancy principle}, 
a classical approach in inverse problems; see, for example, 
\mycite{Phillips1962ATF}, \mycite{Morozov1966OnTS}, and \mycite{cb70d984112e46269082651d324326e7}.
The discrepancy principle relies on comparing the empirical risk 
$\| Y - \widehat{\mathcal{S}} f_t \|_n^2$ 
(also referred to as the squared residual) 
with the noise level 
$\mathbb{E}_\epsilon \| \epsilon \|_n^2 = \sigma^2$, where $\mathbb{E}_\epsilon(\cdot)= \mathbb{E}(\cdot|x_1,\dots,x_n)$ denotes the expectation with respect to $(x_1,y_1),\dots,(x_n,y_n)$ conditional on the design $(x_1,\dots,x_n)$. 
The central idea is to choose a value of $t$ for which 
these two quantities are of comparable magnitude.

To motivate this principle, \mycite{celisse:hal-02548917} employ a bias--variance decomposition and show that balancing
\(\mathbb{E}_{\epsilon}\|Y-\widehat{\mathcal S}f_t\|_n^2\)
with the noise level \(\sigma^2\) yields optimal convergence rates for the excess risk.
However, the raw residual \(\|Y-\widehat{\mathcal S}f_t\|_n^2\) is typically too noisy to provide an accurate estimate of
\(\mathbb{E}_{\epsilon}\|Y-\widehat{\mathcal S}f_t\|_n^2\). Indeed, even under the idealized assumption \(f_t=g^\ast\), one has
\[
\bigl|\|Y-\widehat{\mathcal S}f_t\|_n^2-\mathbb{E}_{\epsilon}\|Y-\widehat{\mathcal S}f_t\|_n^2\bigr|
=
\bigl|\|\epsilon\|_n^2-\sigma^2\bigr|.
\]
For the latter quantity, Proposition~\ref{prop:lownoise} implies that for Gaussian noise
\(\epsilon\sim\mathcal N(0,\sigma^2 I_n)\),
\[
\mathbb P\!\left(
\bigl|\|\epsilon\|_n^2-\sigma^2\bigr|
\;\ge\; \frac{\sigma^2}{\sqrt n}
\right)
\;\ge\; \frac{1}{60}.
\]
Consequently, the residual \(\|Y-\widehat{\mathcal S}f_t\|_n^2\) cannot, in general, estimate its expectation more accurately than order \(O(n^{-1/2})\).
Moreover, \mycite{blanchard2017optimaladaptationearlystopping} showed for the spectral cut-off estimator that early-stopping rules based directly on the residual
\(\|Y-\widehat{\mathcal S}f_t\|_n^2\) can achieve at best a convergence rate of order \(O(n^{-1/2})\).

\medskip
\textbf{Smoothed Discrepancy Principle.}
To obtain optimal convergence rates, \mycite{articlemathe} proposed to
smooth the residual by applying a contraction matrix $L$, a technique
originating from the theory of inverse problems. The underlying idea is
that
\[
\| L (Y - \widehat{\mathcal{S}} f_t) \|_n^2
\]
provides a more stable estimator of
\[
\mathbb{E}_\epsilon
\| L (Y - \widehat{\mathcal{S}} f_t) \|_n^2,
\]
since the operator $L$ reduces fluctuations induced by the noise.
Different choices of $L$ have been studied, 
including $L = K^{q/2}$ with $q \le 1$ by \mycite{stankewitz2020smoothedresidualstoppingstatistical}, 
and more recently 
$L_t := \phi_t^{1/2}(K) K^{1/2}$ by \mycite{celisse:hal-02548917}.

In the latter work, the stopping time $\tau$ is defined by multiplying 
both the residual and the noise by $L_{T}$ as follows:
\begin{align}
\label{martstop}
\tau := \min\!\left\{ 
    t \in \mathbb{N} \,\bigg|\, 
    \| L_T (Y - \widehat{\mathcal{S}} f_t) \|_n^2  
    \le \frac{\sigma^2 \widehat{\mathcal{N}}^{\phi}(T)}{n}
\right\}\wedge T,
\end{align}
for some $T > 0$, where $\widehat{\mathcal{N}}^{\phi}(T) = \operatorname{tr}(L_{T}^\top L_{T})$ denotes the empirical effective dimension.
To the best of our knowledge, it has not been shown for any stopping time of the form~\eqref{martstop}
that it can be simultaneously adaptive over a wide range of target regularities.
Concretely, throughout this paper we quantify regularity via standard source conditions
with respect to the kernel integral operator~$\Sigma$ (cf.~Assumption~\ref{ass:source}).
The difficulty lies in the sensitivity of the rule to the degree of smoothing: 
inaccurate estimates of the hyperparameter $T$ 
can lead to suboptimal stopping times. 
Too large or too small choices of $T$ result in premature stopping, as illustrated in Figure~\ref{smooth}.

\begin{figure}[t]
    \centering
    \includegraphics[width=0.6\linewidth]{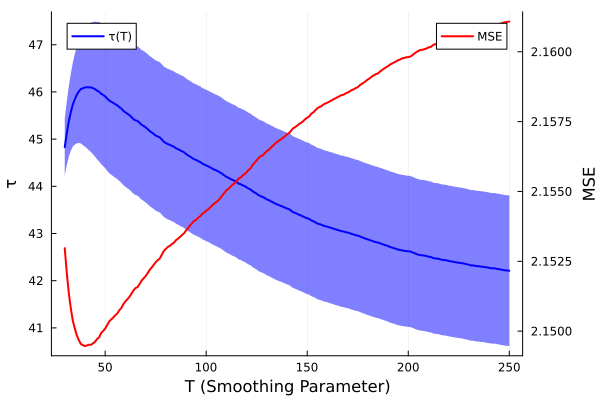}
    \caption{%
        Mean squared error of $f_{\tau(T)}$ (MSE, red) and stopping time $\tau(T)$ (blue)
        as functions of the smoothing parameter $T$.
        Shaded areas indicate $\pm$ one standard deviation across trials.
    }
    \label{smooth}
\end{figure}

\medskip
\textbf{Multi-Scale Smoothed Discrepancy Principle.}
To prevent premature stopping, we modify the above rule by requiring 
the residual criterion to hold simultaneously for multiple smoothing parameters. 
More precisely, we choose $T$ on a logarithmic grid $\mathcal{T}$ 
and stop the algorithm when the smoothed residual 
is smaller than the smoothed variance for all $T \in \mathcal{T}$:
\begin{align} \label{multiscale}
\tau := \min \Bigl\{ 
    t \in \mathbb{N} \,\Big|\, 
     \| L_T (Y - \widehat{\mathcal{S}} f_t) \|_n^2
    \le \frac{\sigma^2 \widehat{\mathcal{N}}^{\phi}(T)}{n}, 
    \quad \forall\, T \in \mathcal{T}_t 
\Bigr\},
\end{align}
where $\mathcal{T}_t := \{\, t e^p : p = 0, \dots, \lfloor \log(n/t) \rfloor \,\}$.

To determine the empirical effective dimension, typically $O(n^3)$ operations are required. 
Since the Gram matrix is symmetric, it is natural to compute its SVD, 
which also has a time complexity of $O(n^3)$. 
Once the eigenvalues have been computed, 
the remaining steps are only of order $O(n^2 \log(n))$. 
Thus, the overall computational complexity of our stopping time \eqref{multiscale}
is of the same order as that of \mycite{celisse:hal-02548917} \eqref{martstop}. 
However, since $O(n^3)$ becomes prohibitively slow for large datasets, this highlights the need for acceleration techniques such as random feature approximation (RFA), which can reduce runtime to $O(n^2 \log n)$.

\medskip
{\bf Random Feature Approximations.}
The idea of RFA is to approximate kernels that admit an integral representation \eqref{eq:kernel-rep} by a kernel represented by a finite sum, i.e., $K(x,\tilde{x}) \approx K_M(x,\tilde{x})$ for $x,\tilde{x} \in \mathcal{X}$ and $M \in \mbn$, where

\begin{align}
 K_M(x,\tilde{x}):=\sum_{i=1}^p \Phi_M^{(i)}(x)^\top \Phi^{(i)}_M(\tilde{x})\,= \sum_{i=1}^p \frac{1}{M} \sum_{m=1}^M 
\varphi_i(x,\omega_m) \cdot \varphi_i(\tilde{x},\omega_m). \label{kernelapprox}
\end{align}

Here $\Phi_M^{(i)}: \mathcal{X} \rightarrow \mathbb{R}^M$ ,  $\Phi_M^{(i)}(x)=M^{-1/2}(\varphi_{i}(x,\omega_1), \dots, \varphi_{i}(x,\omega_M))$ is a finite dimensional feature map,  $\varphi^{(i)}:\mathcal{X}\times\Omega\rightarrow \mathbb{R}$ for some probability space $(\Omega,\pi)$, 
and $\{\omega_m\}_{m=1}^M$ are drawn i.i.d.\ from $\pi$. The associated RKHS is denoted by $\cH_M$. Under Assumption \eqref{eq:RFbound}, the inclusion $\cS_M: \cH_M \hookrightarrow L^2(\mathcal{X}, \rho_{x}; \mathcal{Y})$ is a bounded linear map. 

\begin{assumption}[Kernel]
\label{ass:kernel}
Assume that the kernel $K$ admits an integral representation of the form 
\begin{equation}
\label{eq:kernel-rep}
K(x,\tilde{x}) = \sum_{i=1}^p \int_\Omega \varphi_i(x,\omega)\cdot \varphi_i(\tilde{x},\omega)\, d\pi(\omega),
\end{equation}
where $\varphi_i : \mathcal{X} \times \Omega \to \mathbb{R}$ for $i=1,\dots,p$, and $(\Omega,\pi)$ is a probability space.  
Moreover, assume for all $x \in \mathcal{X}$, 
\begin{equation}
\label{eq:RFbound}
\sum_{i=1}^p |\varphi_i(x,\omega)|^2 \leq \kappa^2 
\quad \pi-\text{almost surely}. 
\end{equation}
\end{assumption}
Note that Assumption~\ref{ass:kernel} implies $|K(x,\tilde{x})| \leq \kappa^2$ for all $x,\tilde{x}\in\mathcal{X}$. Examples of \eqref{eq:kernel-rep} include the \emph{Gaussian kernel} and \emph{random Fourier features} \mycitep{NIPS2007_013a006f, features, nguyen2023random}. The random feature estimator has the same form as the standard kernel estimator~\eqref{def:estimator}, except that $K$ is replaced by $K_M$:
\begin{align}
\label{def:estimator2}
\widehat{\mathcal{S}}_M f_t^M
= \widehat{\mathcal{S}}_M \, \phi_t(\widehat{\Sigma}_M) \, \widehat{\mathcal{S}}^{*}_M Y
= \phi_t(K_M)\, K_M\, Y
\;\in\; \mathbb{R}^n,
\end{align}
where $\widehat{\Sigma}_M := \widehat{\mathcal{S}}_M^{*} \widehat{\mathcal{S}}_M$, and
$\widehat{\mathcal{S}}_M : \mathcal{H}_M \to \mathbb{R}^n$,  $\,\widehat{\mathcal{S}}^{*}_M : \mathbb{R}^n \to \mathcal{H}_M$ are defined by
\begin{align*}
\bigl(\widehat{\mathcal{S}}_M f\bigr)_{j} 
    = \left\langle f, K_M(x_{j}, \cdot) \right\rangle_{\mathcal{H}_M}, \quad
\widehat{\mathcal{S}}^{*}_M Y 
    = \frac{1}{n} \sum_{i=1}^n K_M(x_i, \cdot)\, y_i.
\end{align*}

The main advantage of RFA is that it scales kernel methods to large datasets by reducing both computational and memory costs, while retaining their statistical guarantees \mycitep{features, nguyen2023random}.
Other well-known kernel sketching techniques include the \emph{Nyström method} \mycitep{williams2001using,nystrom2009smola} and \emph{random projection sketching} \mycitep{alaoui2015fast,clarkson2017low}, both of which approximate the Gram matrix directly.
However, a key conceptual advantage of RFA is that it provides a bridge between kernel estimators and neural networks via the neural tangent kernel (NTK) regime \mycitep{jacot2018neural, lee2019wide, nguyen2023neurons}.
In Section~\ref{NNs}, we exploit this connection to show that our stopping rule for random feature methods extends directly to neural networks.

\medskip
\textbf{Smoothed Discrepancy Principle with Random Features.}

Following the same principle as in \eqref{multiscale}, we now aim to stop the random feature estimator $f_t^M$ at an appropriate iteration $t$. 
In contrast to~\eqref{multiscale}, we additionally need to determine the number of random features~$M$
required to achieve optimal statistical guarantees.
To this end, we search over $M$ on a logarithmic grid, leading to the algorithm described in~\eqref{algom}.

\begin{algorithm}[H]
\caption{Adaptive early stopping with random features}
\label{algom}

\textbf{Input:}\vspace*{-0.9cm}\\

\hspace*{-4.5cm}
\begin{minipage}{1.2\linewidth}
\begin{align*}
\mathcal{M} &:= \{\lfloor n^{\frac{j}{\log n}} \rfloor \mid j = 1, \dots, \log n \}, 
 \text{ constants } \tilde{c}, c > 0, \\
 \mathcal{T}_t &:= \{ t e^p : p = 0, \dots, \lfloor \log(n/t) \rfloor \},\\ 
 L_T^m &:= \phi_T^{1/2}(K_m) K_m^{1/2},\\
V_{T,m} &:= \sigma^2 \max\{\widehat{\mathcal{N}}_{m}^{\phi}(T), \log\log n\}, \\
\,\widehat{\mathcal{N}}^{\phi}_m(T) &:= \operatorname{tr}(\phi_T(K_m) K_m).
\end{align*}
\end{minipage}

\vspace*{0.4cm}

\textbf{For} $m \in \mathcal{M}$ \textbf{do:}\vspace*{-0.3cm}
\begin{align}\label{constants}\nonumber
& \textbf{Set } \tau_{m,c} := \min\left\{t \in \mathbb{N} \,\Big|\,  
   \| L_T^m(Y - \widehat{\mathcal{S}}_m f_t^m) \|_n^2
   \le \tfrac{c V_{T,m}}{n},\quad \forall T \in \mathcal{T}_t \right\}\wedge m,\\[3pt]
& \textbf{Break if }
  \Bigl(\tau_{m,c} < m \text{ and } 
  \tfrac{\tilde{c}\log n}{m} \le  
  \tfrac{\widehat{\mathcal{N}}_m^{\phi}(\tau_{m,c})}{n}\Bigr),\\[7pt]
& \textbf{Set } M := m \text{ and } \tau_c := \tau_{M,c}.\nonumber
\end{align}

\textbf{End.}
\end{algorithm}


The constants $\tilde{c}, c > 0$ from \eqref{constants} are introduced for technical convenience, 
although we conjecture that our results also hold for $\tilde{c} = c = 1$. 
In our simulations in Section~\ref{sec:numerics}, 
we likewise fix $\tilde{c} = c = 1$.

The computation time of the above algorithm using RFA 
is reduced from $O(n^3)$ to $O(n M^2 \log n)$. 
The key idea is that the approximated Gram matrix $K_M$ 
can be written as $K_M = \Phi_M \Phi_M^\top$, 
where $\Phi_M \in \mathbb{R}^{n \times M}$ with $(\Phi_M)_{j,m}= \frac{1}{\sqrt{M}}\varphi(x_j,\omega_m)$ and $\varphi$ from \eqref{kernelapprox}.  
Hence, instead of computing the SVD of $K_M$, 
we compute the SVD of $\Phi_M$, 
which requires only $O(n M^2)$ operations.  
As shown in \mycite{nguyen2023random}, 
it suffices to choose $M = O(d \sqrt{n})$. 
Remarkably, our simulations in Section~\ref{sec:numerics} 
show that the algorithm typically terminates at $M \approx d \sqrt{n}$, and the test error plateaus for larger values of $M$.

\section{Main Results}
\label{sec:main-results}

\subsection{Assumptions and Main Results}

In this section, we introduce the assumptions underlying our analysis and present the main theoretical results. We begin by imposing a standard sub-Gaussian condition on the data distribution \mycitep{vershynin2018high, celisse:hal-02548917}.

\begin{assumption}[Data Distribution]
\label{ass:subgaus}
There is a constant $Q \geq 1$ such that

$$
\forall q \geq 1, \quad q^{-1 / 2}\left(\mathbb{E}\left(|\epsilon|^q \mid x\right)\right)^{1 / q} \leq Q \sigma.
$$

\end{assumption}

To characterize the smoothness of $g^*$ relative to the kernel, 
we impose a so-called \emph{source condition}. This condition links $g^*$ to the spectral properties of the kernel integral operator and plays a central role in determining the attainable learning rates.

Denote by $\Sigma:  \mathcal{H}\to  \mathcal{H}$ the kernel integral operator associated to $K$, i.e. 
\[\Sigma g = \int_{\mathcal{X}}  K(x,.)\, g(x)\;\rho_x(d x) . \]

\begin{assumption}[Source Condition]
\label{ass:source}
Let $R>0$, $s\geq0$. We assume 
\begin{align}
g^*  = \Sigma ^s h \;, \label{hsource}
\end{align}
for some $h \in \mathcal{H }$, satisfying $\|h\|_{\mathcal{H }} \leq R$ . 
\end{assumption}

This assumption characterizes the hypothesis space and relates to the regularity of the regression function $g^*$. The larger $s$ is, the smaller the hypothesis space is, the stronger the assumption is, and the easier the learning problem is, as $\text{ Im}(\Sigma ^{s_{1}}) \subseteq \text{ Im}(\Sigma ^{s_{2}})$ if $s_{1} \geq s_{2}$.

\begin{assumption}[Effective Dimension]
\label{ass:dim} For some $b \in[0,1]$ and $C_b>0, \Sigma$ satisfies
\begin{align}
 \mathcal{N}(t):=\operatorname{tr}\left(\Sigma (\Sigma+t^{-1} I)^{-1}\right) \leq C_{b} t^{b}, \quad \text { for all } t>0. \label{effecDim}
\end{align}
\end{assumption}

The left hand-side of \eqref{effecDim} is called effective dimension or degrees of freedom \mycitep{Caponetto}. It is related to covering/entropy number conditions \mycitep{steinwart2008support}. The condition \eqref{effecDim} is naturally satisfied with $b=1$, since $\Sigma$ is a trace class operator which implies that its eigenvalues $\left\{\mu_{i}\right\}_{i}$ satisfy $\mu_{i} \lesssim i^{-1}$. Moreover, if the eigenvalues of $\Sigma$ satisfy a polynomial decaying condition $\mu_{i} \sim i^{-c}$ for some $c>1$, or if $\Sigma$ is of finite rank, then the condition \eqref{effecDim} holds with $b=1 / c$, or with $b = 0$. The case $b = 1$ is referred to as the capacity independent case. A smaller $b$ allows deriving faster convergence rates for the studied algorithms. 

\medskip
The following result establishes that, under the source and capacity assumptions, 
our stopping time achieves the minimax-optimal convergence rate. 
The proof is provided in Appendix \ref{appA}.

\begin{theorem}\label{L2norm}
Suppose that Assumption  \ref{LFL} ,  \ref{ass:subgaus}, \ref{ass:kernel}, \ref{ass:source}, \ref{ass:dim}  hold. Let $\tilde{c},c,\tau_c,\,M$ be defined as in \eqref{algom},  then for all $r:=s+1/2$, we have with probability at least $1-\delta$,

\begin{align*}
\|\mathcal{S}g^*- \mathcal{S}_Mf_{\tau_c}^M\|^2_{L^2(\rho_x)}   \leq \,C n^{-\frac{2r}{2r+b}}\log^4 (1/\delta),
\end{align*}

with $C,c,\tilde{c}>0$ depending on $\sigma^2,r,b,R,E,Q,\kappa,c_p$.

\end{theorem}

\paragraph{Discussion.}
Theorem \ref{L2norm} establishes minimax-optimal convergence rates for the proposed data driven early stopping rule in the $L^2$-norm. Importantly, this result extends the theoretical guarantees of \mycite{celisse:hal-02548917}, originally developed for kernel methods in their full Gram matrix form, to the setting of random feature approximation. A key novelty of our analysis is that it does not rely on any explicit eigenvalue decay assumptions for the empirical kernel Gram matrix, nor does it require continuity assumptions on the regularization function. Most notably, our stopping rule is adaptive with respect to the unknown smoothness $s > 0$ of the target function and the effective dimension parameter $b \in [0,1]$. This adaptivity allows the procedure to automatically achieve optimal rates across a broad spectrum of regularity scenarios without requiring prior knowledge of these parameters.

\section{Applications to Neural Networks}
\label{NNs}

In this section, we demonstrate how our results extend to derive convergence rates for neural networks in the Neural Tangent Kernel (NTK) regime. 
We begin by recalling the definitions of neural networks and the associated NTK, following the setting of \citet{nguyen2023neurons} in summarized form.

\paragraph{Two-Layer Neural Network.}

We consider the class of two-layer neural networks of width $M \in \mathbb{N}$, defined as
\begin{eqnarray*}
\cF_{M} & := \left\{ g_\theta :\cX \to \mbr \; : \; g_\theta(x) =  
\frac{1}{\sqrt M} \sum_{m=1}^M a_m \varsigma( \inner{b_m , x} + \gamma c_m) \;, \right. \\
& \quad \left. \theta =(a, B, c) \in \mbr^M \times \mbr^{d \times M} \times \mbr^M \;, \gamma \in [0,1] \right\}\;. 
\end{eqnarray*}

For the activation function $\varsigma$, we adopt the following regularity condition from \citet{nguyen2023neurons}.

\begin{assumption}[Activation Function]
\label{ass:neurons}
There exists $C_\varsigma > 0$ such that 
$\|\varsigma'\|_\infty \le C_\varsigma$, 
$\|\varsigma''\|_\infty \le C_\varsigma$, and 
$|\varsigma(u)| \le 1 + |u|$ for all $u \in \mathbb{R}$.
Moreover, $\varsigma''$ is assumed to be Lipschitz continuous.
\end{assumption}

\paragraph{Gradient Descent and Initialization.}
Following \citet{nguyen2023neurons}, we train the network parameters by gradient descent on the empirical loss:
\begin{align}
\theta_{t+1}^j 
&= \theta_t^j 
- \alpha \, \partial_{\theta^j} \widehat{\mathcal{E}}(g_{\theta_t})
= \theta_t^j 
- \frac{\alpha}{n} \sum_{i=1}^{n} 
\bigl( g_{\theta_t}(x_i) - y_i \bigr)
\, \partial_{\theta^j} g_{\theta_t}(x_i),
\label{eq:GD}
\end{align}
where $\alpha > 0$ denotes the step size.

\cite{nguyen2023neurons} employ a symmetric initialization scheme for the network parameters $\theta_{0}$ to ensure that $g_{\theta_{0}} \equiv 0$. Importantly, this symmetric trick does not affect the limiting neural tangent kernel (NTK); see \citet{zhang2020type} for details. 

Specifically, the weights in the output layer are initialized symmetrically as
\[
a_{m}^{(0)} = \tau \quad \text{for } m = 1, \dots, M/2,
\qquad 
a_{m}^{(0)} = -\tau \quad \text{for } m = M/2+1, \dots, M,
\]
where $\tau > 0$ is a fixed constant. The input layer parameters are initialized in a coupled manner,
\[
b_{m}^{(0)} = b_{m+M/2}^{(0)} 
\quad \text{for } m \in \{1, \ldots, M/2\},
\]
where the first half of the parameters $\{b_{m}^{(0)}\}_{m=1}^{M/2}$ are drawn independently from the initialization distribution $\pi_{0}$.  The bias parameters are initialized as $c_{m}^{(0)}=0$ for $m \in\{1, \ldots, M\}$.

\paragraph{Neural Tangent Kernel.}
The connection between kernel methods and neural networks is established via the Neural Tangent Kernel (NTK) \citep{jacot2018neural, lee2019wide}.  
The gradient of $g_\theta$ with respect to $\theta$ at initialization $\theta_0$ defines the feature map
\[
\Phi_M(x) := \nabla_\theta g_\theta(x) \big|_{\theta = \theta_0}.
\]
This induces the kernel
\begin{align}
\label{eq:NTK-finite}
K_M(x, x') 
&= \langle \Phi_M(x), \Phi_M(x') \rangle_\Theta \nonumber\\
&= \frac{1}{M} \sum_{r=1}^{M} 
\varsigma\!\bigl(b_r^{(0)\top} x\bigr)
\varsigma\!\bigl(b_r^{(0)\top} x'\bigr)
+ \frac{x^\top x' + \gamma^2}{M} 
\sum_{r=1}^{M} (a_r^{(0)})^2
\varsigma'\!\bigl(b_r^{(0)\top} x\bigr)
\varsigma'\!\bigl(b_r^{(0)\top} x'\bigr).
\end{align}

As $M \to \infty$, the kernel $K_M$ converges to the NTK
\begin{align}
\label{eq:NTK-limit}
K(x, x') 
:= \mathbb{E}_{b^{(0)}}\!\left[
\varsigma(b^{(0)\top} x) \varsigma(b^{(0)\top} x')
\right]
+ \tau^2 (x^\top x' + \gamma^2) \,
\mathbb{E}_{b^{(0)}}\!\left[
\varsigma'(b^{(0)\top} x) \varsigma'(b^{(0)\top} x')
\right].
\end{align}

\paragraph{Stopping Rule for Neural Networks.}

Analogous to the random feature setting, we define 
$\widehat{\mathcal{S}}_Mf_t^M :=  \phi_t(K_M) K_M Y$ as the NTK-based kernel estimator, with $K_M$ from \eqref{eq:NTK-finite}, and 
where $\phi_t(K_M) := \alpha \sum_{s=0}^{t-1}(I - \alpha K_M)^s$ denotes the gradient descent regularization function.

To achieve optimal statistical guarantees, our stopping rule is modified only by a logarithmic factor in $M$.

\begin{algorithm}[H]
\caption{Adaptive early stopping for NNs}
\label{algom2}

\textbf{Input:}\vspace*{-0.9cm}\\

\hspace*{-4.5cm}
\begin{minipage}{1.2\linewidth}
\begin{align*} \mathcal{M} := &\{\lfloor n^{j / \log n} \rfloor \mid j = 1, \dots, \log n \}, 
\text{ constants }\tilde{c}, c > 0, \\
\mathcal{T}_t :=& \{ t e^p : p = 0, \dots, \lfloor \log(n/t) \rfloor \},\\
 L_T^m :=& \phi_T^{1/2}(K_m) K_m^{1/2},\\
V_{T,m} := &\sigma^2 \max\{\widehat{\mathcal{N}}_m^{\phi}(T), \log\log n\}, \\
\widehat{\mathcal{N}}^{\phi}_m(T) :=&\operatorname{tr}(\phi_T(K_m) K_m).
\end{align*}
\end{minipage}

\vspace*{0.4cm}

\textbf{For} $m \in \mathcal{M}$ \textbf{do:}\vspace*{-0.3cm}
\begin{align}
& \textbf{ Set } \tau_{m,c} 
:= \min\Bigl\{ t \in \mathbb{N} \,\big|\,  
\| L_T^m(Y - \widehat{\mathcal{S}}_m f_t^m) \|_n^2
\le \tfrac{c V_{T,m}}{n},\ 
\forall T \in \mathcal{T}_t
\Bigr\} \wedge m, \nonumber \\[3pt]
& \textbf{ Break if } 
\Bigl( \tau_{m,c} < m
\ \text{and}\ 
\tfrac{\tilde{c} \log^8 n}{m} 
\le \tfrac{\widehat{\mathcal{N}}_m^{\phi}(\tau_{m,c})}{n} \Bigr), \label{secstopcrit}\\[7pt]
& \textbf{ Set } M := m \text{ and } \tau_c := \tau_{M,c}. \nonumber
\end{align}

\textbf{End.}
\end{algorithm}


Compared to Algorithm~\ref{algom}, only the second stopping criterion~\eqref{secstopcrit}
is modified by an additional logarithmic factor to ensure that the number of random features~$M$
is sufficiently large.
Under this condition, the neural network is guaranteed to behave approximately as its associated
kernel estimator; see Theorem \ref{prop:second-term} together with Proposition \ref{prop:effecdim5}.
As a consequence, one could equivalently replace the kernel estimator
\(\mathcal S_M f_t^M\) by the neural network \(g_{\theta_t}\) directly in the residual-based stopping rule, that is,
\[
\| L_T^m (Y - g_{\theta_t}) \|_n^2
\;\le\;
\frac{c\, V_{T,m}}{n}.
\]
However, once the eigenvalues required for the empirical effective dimension have been computed,
evaluating the kernel estimator \(\mathcal S_M f_t^M\) is significantly faster in practice than
evaluating the corresponding neural network.
For this reason, we formulate the stopping rule in terms of the kernel estimator.
In addition we require that the effective dimension is also bounded from below, i.e.,

\begin{assumption}[Effective Dimension]
\label{ass:dim2}
For some $b \in [0,1]$ and constants $c_b, C_b > 0$, the operator $\Sigma$ satisfies
\[
c_b t^{b} \le \mathcal{N}(t) \le C_b t^{b}, 
\quad \forall\, t > 0.
\]
\end{assumption}

\begin{theorem}
\label{NNstop}
Suppose Assumptions~\ref{ass:subgaus}, \ref{ass:source}, \ref{ass:neurons}, and \ref{ass:dim2} 
hold with respect to the NTK $K$ from \eqref{eq:NTK-limit}. 
Let $\tilde{c},c,\tau_c, M$ be defined as in Algorithm \eqref{algom2}, and $\alpha < 1/\kappa^2$. 
Then for all $r := s + \tfrac{1}{2}$, with probability at least $1 - \delta$,
\[
\|\mathcal{S} g^* - g_{\theta_{\tau_c}}\|^2_{L^2(\rho_x)}
\le C\, n^{-\frac{2r}{2r + b}} \log^4(1/\delta),
\]
where $C, \tilde{c}, \bar{c}, c > 0$ depend on 
$\sigma^2, r, b,  R, E, Q, \kappa, c_p$.
\end{theorem}

\paragraph{Discussion.}
Theorem~\ref{NNstop} shows that the proposed stopping rule, 
developed for random feature approximations of kernel methods, 
also achieves minimax-optimal convergence rates for overparameterized neural networks in the NTK regime.  
This highlights the versatility of our algorithm, 
bridging classical nonparametric theory and modern deep learning.  
Notably, the method provides a fully data-driven selection of the number of hidden neurons $M$, 
for which optimal convergence guarantees are rigorously established.

\section{Numerical Illustration}
\label{sec:numerics}

We assess the performance of our proposed early stopping rule \eqref{algom} in comparison with the non-random feature stopping criterion \eqref{martstop} of \mycite{celisse:hal-02548917} within kernel models. Specifically, we investigate both approaches under the Gaussian kernel. To this end, we generate a synthetic dataset of \( n = 1000 \) random training and test points \( (x,y) \in \mathbb{R}^3 \times \mathbb{R} \), sampled from a standard normal distribution according to  
\[
y = x \cdot \beta + \sigma \epsilon\,,
\]
where \( \beta \sim \mathcal{N}(0, I_3) \), \( \epsilon \sim \mathcal{N}(0, I_n) \), and \( \sigma = 1 \).  
In addition to this synthetic setting, we also employ the real-world dataset \emph{Concrete Compressive Strength}\footnote{\url{https://archive.ics.uci.edu/dataset/165/concrete+compressive+strength}}, which has input dimension \( d=8 \) and is known to exhibit pronounced nonlinear behavior. This dataset provides a complementary testbed for evaluating the two stopping rules under more realistic conditions.  
For the synthetic experiments, results are averaged over 200 independent repetitions. We denote the stopping times by \( \tau_1 \) and \( \tau_2 \), corresponding to the criteria in \eqref{martstop} and \eqref{algom}, respectively. The parameters are fixed at \( c = \tilde{c} \log(n) = 1 \), and the emergency stop is chosen as recommended in \mycite{celisse:hal-02548917}, namely when $
\hat{T}^* \mathcal{N}^{\phi}(\hat{T}^*) = n$.
Figures \ref{fig1}--\ref{fig4} illustrate that the stopping rule $\tau_1$ from \mycite{celisse:hal-02548917} tends to halt prematurely and exhibits greater variability across repetitions. By contrast, our proposed stopping time mitigates this issue: instead of requiring the residuals to satisfy the criterion under a single smoothing parameter, the rule enforces the condition over a range of smoothing levels that approximately contains the optimal one. This modification results in improved stability and performance.  
Moreover, Figure \ref{fig4} highlights the substantial computational advantage of our method, which is achieved through the use of random feature approximations. Importantly, these benefits manifest consistently across both the synthetic and the real dataset.  
Finally, Figure \ref{fig6} illustrates that our procedure provides a practically meaningful upper bound for determining the number of random features \( M \). Although the theoretically optimal choice of \( M \) may be slightly smaller than the value suggested by our rule, the bound remains effective and can be further refined through a finer grid selection set \( \mathcal{M} \), tailored to the specific problem.

\begin{figure}[H]
    \centering
    \begin{minipage}[t]{0.4\textwidth} 
        \centering
        \includegraphics[width=\linewidth]{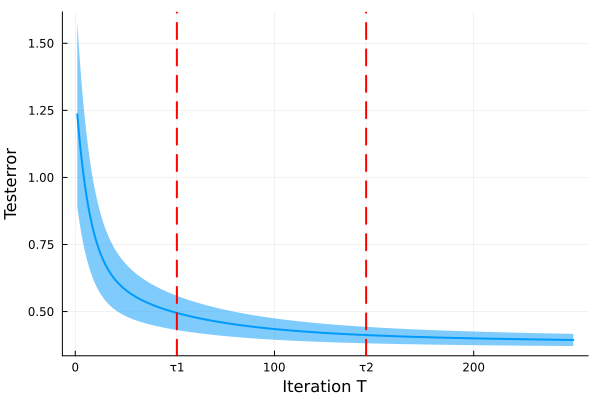}
        \caption{Mean test error over 200 repetitions of the synthetic dataset for different numbers of iterations. The vertical line indicates the average stopping times.\label{fig1}}
    \end{minipage}%
    \hfill
    \begin{minipage}[t]{0.4\textwidth}
        \centering
        \includegraphics[width=\linewidth]{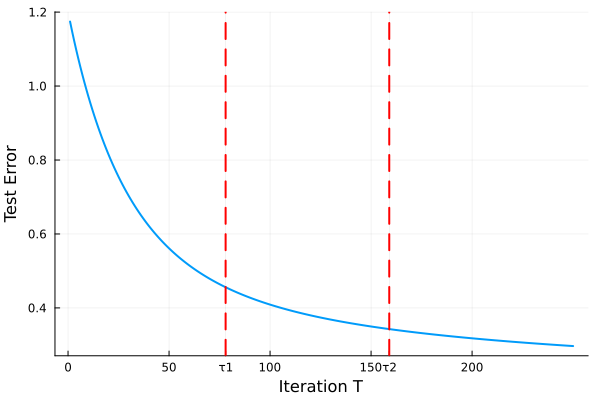}
        \caption{Test error of the \emph{Concrete Compressive Strength} dataset for different numbers of iterations and $n=800$. The vertical line indicates the average stopping times.}
    \end{minipage}
\end{figure}

\begin{figure}[H]
    \centering
    \begin{minipage}[t]{0.4\textwidth}
        \centering
        \includegraphics[width=\linewidth]{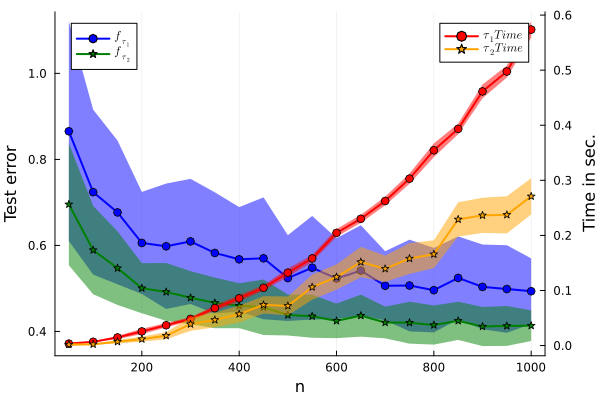}
        \caption{The blue and green lines depict the test error of the kernel estimator, stopped according to the respective stopping rules for different choices of \( n \). The red and yellow lines indicate the computation time required to evaluate the stopping rules. Simulations were performed on the random dataset.}
    \end{minipage}%
    \hfill
    \begin{minipage}[t]{0.4\textwidth}
        \centering 
        \includegraphics[width=\linewidth]{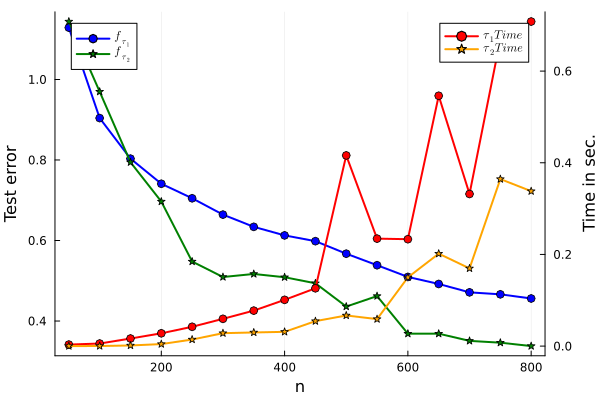}
        \caption{The blue and green lines depict the test error of the kernel estimator, stopped according to the respective stopping rules for different choices of \( n \). The red and yellow lines indicate the computation time required to evaluate the stopping rules. Simulations were performed on the \emph{Concrete Compressive Strength} dataset.\label{fig4}}
    \end{minipage}
\end{figure}

\begin{figure}[H]
    \centering
    \begin{minipage}[t]{0.43\textwidth}
        \centering
        \includegraphics[width=\linewidth]{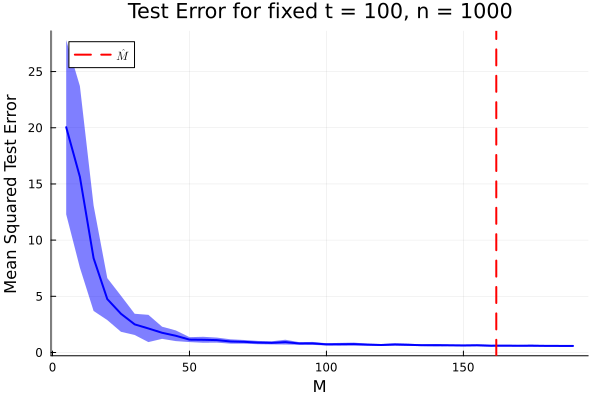}
        \caption{The curve shows the test error for the random dataset across different choices of \( M \). The vertical line indicates our average stopping rule for \( M \).}
    \end{minipage}%
    \hfill
    \begin{minipage}[t]{0.43\textwidth}
        \centering
        \includegraphics[width=\linewidth]{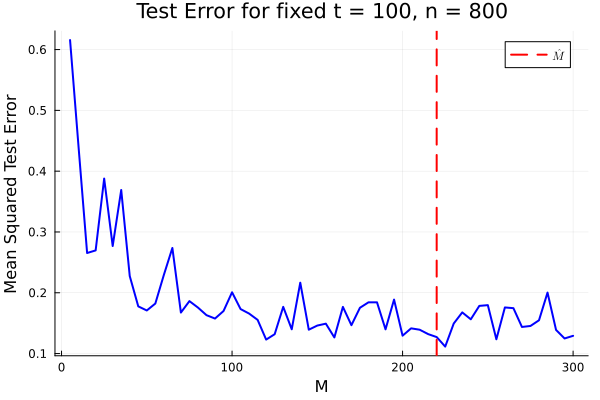}
        \caption{The curve shows the test error for the \emph{Concrete Compressive Strength} dataset across different choices of \( M \). The vertical line indicates our stopping rule for \( M \).\label{fig6}}
    \end{minipage}
\end{figure}

\section*{Conclusion}

We have introduced a fully data-driven early stopping strategy for spectral regularization within the framework of random feature models. By leveraging the computational efficiency of random features, our method bypasses the prohibitive costs associated with full Gram matrix operations in traditional kernel-based learning, thereby ensuring scalability to large-scale datasets.

Our theoretical results establish that this procedure achieves minimax-optimal convergence rates without requiring prior knowledge of the target function's smoothness or the effective dimension of the problem. Notably, the proposed method exhibits adaptivity across a broad range of regularity classes, encompassing all source conditions with smoothness parameter $s \geq 0$ and capacity parameters $b \in [0,1]$.

Furthermore, we demonstrated that our stopping rule extends beyond classical kernel methods into the regime of overparameterized neural networks. In this context, the algorithm provides a data-driven recommendation for the network width—specifically, the number of hidden neurons—required to guarantee optimal learning rates. To the best of our knowledge, this is the first approach to provide such dual adaptivity guarantees within the Neural Tangent Kernel (NTK) framework. Numerical simulations corroborate our theoretical findings, showing that the algorithm consistently identifies stopping times and model sizes that yield near-optimal generalization performance.

Despite these contributions, several promising avenues for future research remain:

\begin{itemize}
    \item \textbf{Architectural Extensions:} Recent work by \cite{li2025optimalratesgeneralizationgradient}, building upon \cite{nguyen2023neurons}, has established minimax rates for deep neural networks with ReLU activations in the NTK regime. By combining these results with the arguments presented in Theorem \ref{NNstop}, we anticipate that our stopping rule can be generalized to deep architectures. Moreover, we expect our framework to extend to operator-valued kernel methods \citep{nguyen2023random, vvk2}. As shown in \cite{nguyen2024optimalconvergenceratesneural}, such methods provide a foundation for analyzing the increasingly prevalent Neural Operators (NOs). We conjecture that our stopping rule can be adapted to achieve optimal rates for NOs within the operator-valued NTK regime.
    
    \item \textbf{Robustness to Heavy-Tailed Noise:} While our current analysis assumes sub-Gaussian noise, many real-world applications involve heavy-tailed distributions. Adapting our residual thresholding mechanism to such robust settings would significantly broaden the framework's practical utility. Given the recent progress in minimax rates for heavy-tailed noise in spectral regularization \citep{mollenhauer2025regularizedsquareslearningheavytailed}, we believe our adaptivity results can be extended to accommodate these more general noise assumptions.
\end{itemize}

\bibliography{bib_iteration}


\section{Proofs of the Main Results}\label{appA}

In this section, we provide the proofs of our main theoretical results.

\paragraph{Notation.}
Throughout the proofs we use the following shorthand notation for $T, M > 0$:
\[
L_T := \phi_T^{1/2}(K) K^{1/2}, 
\qquad  
L_T^M := \phi_T^{1/2}(K_M) K_M^{1/2},
\]
\[
\tilde{g}_T := L_T^M \widehat{\mathcal{S}} g^*, 
\qquad 
\tilde{\epsilon}_T := L_T^M \epsilon.
\]

To estimate the effective dimension 
\(
\mathcal{N}(T) := \operatorname{tr}\!\left(\Sigma (\Sigma+T^{-1}I)^{-1}\right),
\)
we introduce the empirical analogues
\[
\widehat{\mathcal{N}}(T) 
     := \operatorname{tr}\!\bigl(K\bigl(K+T^{-1}I_n\bigr)^{-1}\bigr), 
\qquad
\widehat{\mathcal{N}}_M(T) 
     := \operatorname{tr}\!\bigl(K_M\bigl(K_M+T^{-1}I_n\bigr)^{-1}\bigr),
\]
\[
\widehat{\mathcal{N}}^{\phi}(T) 
     := \operatorname{tr}\!(L_{T}^\top L_{T}), 
\qquad
\widehat{\mathcal{N}}^{\phi}_M(T) 
     := \operatorname{tr}\!\bigl((L_{T}^M)^\top L_{T}^M\bigr).
\]

Furthermore, $C_\bullet > 0$ denotes a universal constant that may vary from line to line, depending only on the parameters \\ 
$\sigma^2, r, b,  R, E, Q, \kappa, c_p, c, \tilde{c}, \bar{c}$, but never on $\delta$, $p$, $M$, or $n$.

\subsection{Proof Organization and Error Decomposition}

The structure of our proofs is as follows.  
Section~\ref{mainproofs} establishes Theorem~\ref{L2norm}.  
We begin with Theorem~\ref{empnorm}, which provides a bound on the empirical excess risk.  
The proof of the main theorem then builds on this result and employs classical concentration inequalities in the kernel setting to transfer bounds from the empirical to the population $L^2$ error.

To bound the empirical error, we use the decomposition
\[
\|\widehat{\mathcal{S}} g^* - \widehat{\mathcal{S}}_M f^M_{\tau_c}\|_n^2
\;\;\le\;\;
2\| r_{\tau_c}(K_M)\widehat{\mathcal{S}} g^* \|_n^2
\;+\;
2\|\tilde{\epsilon}_{\tau_c}\|_n^2,
\]
which we refer to as
\[
\text{Approximation Error} 
\qquad + \qquad 
\text{Estimation Error}.
\]

The approximation error is controlled using spectral calculus, the approximation properties of the residual polynomial $r_t$ \eqref{c_r}, and structural properties of the stopping time~$\tau_c$.

To bound the estimation error, we apply classical concentration inequalities 
(see Proposition~\ref{epsbound}), which yield
\[
\|\tilde{\epsilon}_{\tau_c}\|_n^2 
\;\lesssim\; 
\frac{\widehat{\mathcal{N}}^{\phi}_M(\tau_c)}{n}.
\]

In Section~\ref{upperbound}, we establish an upper bound on the stopping rule by
comparing $\tau_c$ with a theoretically optimal stopping time $\bar{T}_M$, such that
\[
\frac{\widehat{\mathcal{N}}^{\phi}_M(\tau_c)}{n}
\;\lesssim\;
\frac{\widehat{\mathcal{N}}^{\phi}_M(\bar{T}_M)}{n},
\]
where the latter quantity achieves optimal rates by construction of $\bar{T}_M$
(see Proposition~\ref{prop:effecdim5}).

In Section~\ref{oprates nns}, Theorem~\ref{L2norm} is combined with additional NTK arguments to prove optimal rates for neural networks, thereby establishing Theorem~\ref{NNstop}.

All auxiliary operator inequalities and concentration bounds required for the above arguments are collected in Section \ref{appendixB}.


\subsection{Proving the Main Result}
\label{mainproofs}

\begin{theorem}\label{empnorm}
Suppose Assumptions \ref{LFL}, \ref{ass:subgaus}, \ref{ass:kernel}, 
\ref{ass:source}, and \ref{ass:dim} hold. Let 
$\tilde{c}, c, \tau_c, M$ be defined as in \eqref{algom}. Then, for all 
$r := s + 1/2$, with probability at least $1 - \delta$,
\begin{align*}
\|\widehat{\mathcal{S}} g^* - \widehat{\mathcal{S}}_M f_{\tau_c}^M\|_n^2 
\le C\, n^{-\frac{2r}{2r + b}} \log^3(1/\delta),
\end{align*}
where $C, \tilde{c}, c > 0$ depend on 
$\sigma^2, r, b, R, E, Q, \kappa, c_p, \hat{c}$.
\end{theorem}

\begin{proof}[Proof of Theorem \ref{empnorm}]
We begin with the elementary decomposition

\begin{align}\label{starter}
\|\widehat{\mathcal{S}} g^* - \widehat{\mathcal{S}}_M f^M_{\tau_c}\|_n^2
    \le 2\|r_{\tau_c}(K_M)\widehat{\mathcal{S}} g^*\|_n^2
    + 2\|\tilde{\epsilon}_{\tau_c}\|_n^2.
\end{align}

Define
\[
\tilde{T} := \tau_c \vee n^{\frac{1}{2r + b}}\,.
\]
For this theoretical stopping time $\tilde{T}$, consider the decomposition

\[
\|r_{\tau_c}(K_M)\widehat{\mathcal{S}} g^*\|_n^2
 =
\sum_{\tilde{T}\hat{\mu}_j \ge 1}
 r_{\tau_c}^2(\hat{\mu}_j)\,
 \langle \widehat{\mathcal{S}} g^*, \hat{\varphi}_j \rangle_n^2
 +
\sum_{\tilde{T}\hat{\mu}_j \le 1}
 r_{\tau_c}^2(\hat{\mu}_j)\,
 \langle \widehat{\mathcal{S}} g^*, \hat{\varphi}_j \rangle_n^2,
\]
where $(\hat{\mu}_j, \hat{\varphi}_j)$ denote the eigenvalues and eigenvectors of the approximated kernel Gram matrix $K_M$.

By Assumption~\ref{LFL}, we have 
$c_p \le \phi_t(\mu)\mu$ whenever $\mu t > 1$. Therefore,

\begin{align*}
\|r_{\tau_c}(K_M)\widehat{\mathcal{S}} g^*\|_n^2
&\le 
c_p^{-1}\|
   r_{\tau_c}(K_M)\,
   \phi_{\tilde{T}}^{1/2}(K_M)\,
   K_M^{1/2}\widehat{\mathcal{S}} g^*
\|_n^2
+
\sum_{\tilde{T}\hat{\mu}_j \le 1}
 \langle \widehat{\mathcal{S}} g^*, \hat{\varphi}_j \rangle_n^2
\\[4pt]
&:= II + III.
\end{align*}

Substituting this into \eqref{starter} yields

\begin{align}\label{toshow}
\|\widehat{\mathcal{S}} g^* - \widehat{\mathcal{S}} f_{\tau_c}\|_n^2
\le I + II + III,
\end{align}
where
\[
I := 2\|\tilde{\epsilon}_{\tau_c}\|_n^2, 
\qquad 
II := 
c_p^{-1}\|
   r_{\tau_c}(K_M)\,
   \phi_{\tilde{T}}^{1/2}(K_M)\,
   K_M^{1/2}\widehat{\mathcal{S}} g^*
\|_n^2,
\qquad 
III := 
\sum_{\tilde{T}\hat{\mu}_j \le 1}
 \langle\widehat{\mathcal{S}} g^*, \hat{\varphi}_j\rangle_n^2.
\]


\textbf{Bounding I.}

From Proposition~\ref{upboundrf}, we have for 
$c := 2(1 + \hat{c} + \log(1/\delta))$ that, with probability at least 
$1 - \delta$, 
\[
\tau_c \le \bar{T}_M,
\]
where
\begin{align*}
\bar{T}_M 
    := \min_{t \in \mathbb{N}}
        \left\{
            \| r_t(K_M)\widehat{\mathcal{S}} g^* \|_n^2
            \le \frac{\widehat{\mathcal{N}}_M^{\phi}(t)}{n}
        \right\}\wedge n.
\end{align*}

Since $\phi_t$ is monotone increasing in $t$, we obtain

\[
\|\tilde{\epsilon}_{\tau_c}\|_n^2
    \le \|\tilde{\epsilon}_{\bar{T}_M}\|_n^2,
\]

and by Proposition~\ref{epsbound}, with probability at least $1 - \delta$,

\begin{align}\label{nicenoice}
I 
= \|\tilde{\epsilon}_{\tau_c}\|_n^2
\le \|\tilde{\epsilon}_{\bar{T}_M}\|_n^2
\le \frac{\sigma^2 \hat{c}}{n}
    \bigl( \widehat{\mathcal{N}}_M^{\phi}(\bar{T}_M) + \log(1/\delta) \bigr).
\end{align}

From Proposition~\ref{prop:effecdim5}, we further have for some 
$C_\bullet > 0$ depending only on $\sigma^2, E, b, r, \hat{c}$,

\[
\frac{\widehat{\mathcal{N}}_M^{\phi}(\bar{T}_M)}{n}
    \le C_\bullet n^{-\frac{2r}{2r + b}}.
\]

Thus,
\begin{align}
I 
\le C_\bullet\, n^{-\frac{2r}{2r+b}} \log(1/\delta),
\end{align}

for some constant $C_\bullet > 0$ depending on 
$\sigma^2, E, b, r, \hat{c}$.
\medskip

\textbf{Bounding II.}

By construction, $\tau_c \le \tilde{T} < n$.  
From the stopping criterion for $\tau_c$, there exists some 
$T \in \mathcal{T}_{\tau_c}$ such that
\begin{align}\label{TTT}
\frac{T}{e} \le \tilde{T} \le T.
\end{align}

Using this, we have
\begin{align*}
II
&=
\bigl\|
    r_{\tau_c}(K_M)\,
    \phi_{\tilde{T}}^{1/2}(K_M)\,
    K_M^{1/2}\widehat{\mathcal{S}} g^*
\bigr\|_n^2
\\[3pt]
&\le 
2\bigl\|
    r_{\tau_c}(K_M)\,
    \phi_{T}^{1/2}(K_M)\,
    K_M^{1/2} Y
\bigr\|_n^2
\,+\,
2\|\tilde{\epsilon}_{\tilde{T}}\|_n^2
\\[3pt]
&\le 
2\frac{c\,V_T}{n}
+ 2\|\tilde{\epsilon}_{\tilde{T}}\|_n^2,
\end{align*}
where
\[
V_T := \sigma^2 \max\{ \widehat{\mathcal{N}}_M^{\phi}(T),\ \log\log(n)\}.
\]

\textbf{Bounding \(\displaystyle \frac{V_T}{n}\).}

By Assumption~\ref{LFL}, the definition of $\phi$, and \eqref{TTT},

\[
\phi_T(\mu)\,\mu
\le (E \vee 1)\,(T\mu \wedge 1)
\le e(E \vee 1)\,(\tilde{T}\mu \wedge 1)
\le e(E \vee 1)c_\phi\,\phi_{\tilde{T}}(\mu)\mu.
\]

Thus,

\[
\frac{\widehat{\mathcal{N}}_M^\phi(T)}{n}
\le 
C_\bullet\,\frac{\widehat{\mathcal{N}}_M^\phi(\tilde{T})}{n}
\le 
C_\bullet\left(
    \frac{\widehat{\mathcal{N}}_M^\phi(\bar{T}_M)}{n}
    \,\vee\,
    \frac{\widehat{\mathcal{N}}_M^\phi(n^{1/(2r+b)})}{n}
\right),
\]

for some constant $C_\bullet$ depending on $E, c_\phi$.

The first term is bounded by \eqref{nicenoice}.  
For the second term, Propositions \ref{Lemmamartin0}, 
\ref{prop:effecdim3}, and \ref{prop:effecdim2} give, with probability 
at least $1-\delta$,

\begin{align}\label{asdasdff}
\widehat{\mathcal{N}}_M^\phi(n^{1/(2r+b)})
\le 2(E \vee 1)\,\widehat{\mathcal{N}}_M(n^{1/(2r+b)})
\le 
C_\bullet \log^2(1/\delta)\, \mathcal{N}(n^{1/(2r+b)}),
\end{align}

for some $C_\bullet>0$ depending on $E$.

Using Assumption~\ref{ass:dim}, we obtain

\[
\frac{V_T}{n}
\le 
C_\bullet\, n^{-\frac{2r}{2r+b}} \log^2(1/\delta).
\]

\textbf{Bounding $\|\tilde{\epsilon}_{\tilde{T}}\|_n^2$.}

From Proposition~\ref{epsbound}, with probability at least $1-\delta$,

\begin{align}
\|\tilde{\epsilon}_{\tilde{T}}\|_n^2
\le 
\sigma^2 \hat{c}
\left(
    \frac{\widehat{\mathcal{N}}_M^\phi(\tau_c)}{n}
    \,\vee\,
    \frac{\widehat{\mathcal{N}}_M^\phi(n^{1/(2r+b)})}{n}
    + \frac{\log(1/\delta)}{n}
\right)
\le 
C_\bullet\,n^{-\frac{2r}{2r+b}}\log^2(1/\delta),
\end{align}

where both effective-dimension terms were bounded above using 
\eqref{nicenoice} and \eqref{asdasdff}.
Combining the previous bounds yields

\[
II 
\le 
2\frac{c\,V_T}{n}
+ 2\|\tilde{\epsilon}_{\tilde{T}}\|_n^2
\le 
C_\bullet\, n^{-\frac{2r}{2r+b}}\log^3(1/\delta),
\]

for a constant $C_\bullet > 0$ depending on 
$\sigma^2, E, b, r, c_\phi$.

\textbf{Bounding III.}

For any $\tilde{t} > 0$, we have for $K_{M,\tilde{t}}:=(K_M+\tilde{t}^{-1}I)$,

\begin{align}
III
\nonumber&=
\sum_{\tilde{T}\hat{\mu}_j \le 1} \langle \widehat{\mathcal{S}} g^*, \hat\varphi_j \rangle_n^2
\\
&\le
\bigl( \tilde{T}^{-2(r \vee 1)} + \tilde{t}^{-2(r \vee 1)} \bigr)
\sum_{\tilde{T}\hat{\mu}_j \le 1}
    \langle 
        K_{M,\tilde{t}}^{-(r \vee 1)} \widehat{\mathcal{S}} g^*, 
        \hat\varphi_j 
    \rangle_n^2
\nonumber\\
&\le
\bigl( \tilde{T}^{-2(r \vee 1)} + \tilde{t}^{-2(r \vee 1)} \bigr)
\bigl\|
    K_{M,\tilde{t}}^{-(r \vee 1)} \widehat{\mathcal{{S}}} g^*
\bigr\|_n^2.\label{IIIcalc}
\end{align}

Using $g^* = \Sigma^{s} h$ with $\|h\|_{\mathcal{H}} \le R$ and $r = s + 1/2$, we obtain

\begin{align}
\nonumber \bigl\|
    K_{M,\tilde{t}}^{-(r \vee 1)} \widehat{\mathcal{S}} g^*
\bigr\|_n^2
&\le
R^2 \,
\bigl\|
    K_{M,\tilde{t}}^{-(r \vee 1)} K_{\tilde{t}}^{r}
\bigr\|
\,
\bigl\|
    K_{\tilde{t}}^{-r} \widehat{\mathcal{S}} \Sigma^{s}
\bigr\|_n^2
\\
&\le
R^2 \,
\bigl\|
    K_{M,\tilde{t}}^{-(r \vee 1)} K_{\tilde{t}}^{r}
\bigr\|
\,
\bigl\|
    \widehat{\Sigma}_{\tilde{t}}^{-s} \Sigma_{\tilde{t}}^{s}
\bigr\|^2.\label{twonormss}
\end{align}

By choosing 
\[
\tilde{t}:= C_{\bullet}\,\,((\log(n/\delta))^{-1/(2r)}\,
 M^{1/2r}) \wedge n^{\frac{1}{2r+b}}),
\]

we obtain for the first norm from Proposition \ref{OPbound7} that 

\[ \left\|K_{M,\tilde{t}}^{-(r\vee1)}K_{\tilde{t}}^{r}\right\| \leq 3\tilde{t}^{(1-r)^+}.\]

For the last norm in \eqref{twonormss}, observe that for $s=0$ it is trivially bounded by $1$.
Let now $s>0$. By the choice of $\tilde t$, the assumptions of
Proposition~\ref{OPbound8} are satisfied, for a constant $C_{\bullet}$
depending on $s$ and $b$, since
\[
    n^{\frac{1}{2r+b}}
    =
    o\!\left(
        \log^{-3}(n)
        \max\left\{
            n^{\frac{1}{2s}},
            n^{\frac{1}{1+b}}
        \right\}
    \right).
\]
Hence, Proposition~\ref{OPbound8} yields
\[
    \left\|
        \widehat{\Sigma}_{M,\tilde t}^{-s}
        \Sigma_{M,\tilde t}^{s}
    \right\|
    \leq 2 .
\]

Plugging these bounds into \eqref{IIIcalc} yields

\begin{align*}
III
\le
C_\bullet \bigl( \tilde{T}^{-2(r \vee 1)} + \tilde{t}^{-2(r \vee 1)} \bigr)
\tilde{t}^{2(1-r)^+},
\end{align*}

for some constant $C_\bullet > 0$.

Using the definitions of $\tilde{t}$ and $\tilde{T}$, we obtain

\begin{align*}
III
&\le
C_\bullet\left(
    \tilde{T}^{-2(r \vee 1)} n^{\frac{2(1-r)^+}{2r + b}}
    + 
    \Bigl(
        \frac{\log(n/\delta)}{M}
        \vee
        n^{-\frac{2r}{2r+b}}
    \Bigr)
\right)
\\
&\le
C_\bullet\left(
    n^{-\frac{2r}{2r+b}}
    +
    \frac{\log(n/\delta)}{M}
\right),
\end{align*}

for some $C_\bullet>0$ depending on $\kappa, r, p, R$.

Using the definition of $M$, Proposition~\ref{upboundrf}, and Proposition~\ref{prop:effecdim5}, we further have

\begin{align*}
\frac{\log(n/\delta)}{M}
\le
C_\bullet \frac{\widehat{\mathcal{N}}_M^{\phi}(\tau_c)}{n} \log(1/\delta)
\le
C_\bullet \frac{\widehat{\mathcal{N}}_M^{\phi}(\bar{T}_M)}{n} \log(1/\delta)
\le
C_\bullet n^{-\frac{2r}{2r+b}} \log(1/\delta).
\end{align*}

Thus,
\[
III
\le
C_\bullet\, n^{-\frac{2r}{2r+b}} \log(1/\delta).
\]

Collecting the bounds for $I$, $II$, and $III$ proves the claim.

\end{proof}


\begin{proof}[Proof of Theorem \ref{L2norm}]
\label{mainerrorproof}


We first define

\[
f^*_t := \mathcal{S}_M^* \, \phi_t(\mathcal{L}_M)\, \mathcal{S} g^* \in \mathcal{H}_M,
\qquad 
t := \log^{-1/(2r)}(n/\delta)\, M^{1/(2r)},
\]

and begin with the decomposition

\begin{align}
\label{sasdpijfif}
\|\mathcal{S} g^* - \mathcal{S}_M f_{\tau_c}^M\|^2_{L^2(\rho_x)}
\;\leq\;
\|\mathcal{S} g^* - \mathcal{S}_M f^*_t\|^2_{L^2(\rho_x)}
\;+\;
\|\mathcal{S}_M f^*_t - \mathcal{S}_M f_{\tau_c}^M\|^2_{L^2(\rho_x)}.
\end{align}

From Proposition~\ref{rfbiasbound}, we obtain for the first term

\begin{align}
\label{dsjksjdksd}
\|\mathcal{S} g^* - \mathcal{S}_M f^*_t\|^2_{L^2(\rho_x)}
\leq C_\bullet\, t^{-2r} \log^2(1/\delta).
\end{align}

Plugging in the definition of \(t\) and \(M\) yields
\[
\|\mathcal{S} g^* - \mathcal{S}_M f^*_t\|^2_{L^2(\rho_x)}
\leq
C_\bullet \, \frac{\widehat{\mathcal{N}}_M^\phi(\tau_c)}{n}\, \log^2(1/\delta).
\]
Using Proposition~\ref{upboundrf} together with Proposition~\ref{prop:effecdim5},

\[
C_\bullet \frac{\widehat{\mathcal{N}}_M^\phi(\tau_c)}{n}\log^2(1/\delta)
\;\leq\;
C_\bullet\, n^{-\frac{2r}{2r+b}}\, \log^3(1/\delta).
\]

It therefore remains to bound the second term in \eqref{sasdpijfif}.  
For some \(\tilde{t} > 0\), we begin with

\begin{align*}
\|\mathcal{S}_M f^*_t - \mathcal{S}_M f_{\tau_c}^M\|^2_{L^2(\rho_x)}
\leq
\left\|
\Sigma_{\tilde{t}}^{1/2}
\widehat{\Sigma}_{\tilde{t}}^{-1/2}
\right\|
\left(
\left\|\widehat{\Sigma}^{1/2}(f^*_t - f_{\tau_c}^M)\right\|_{\mathcal{H}}^2
+
\frac{1}{\tilde{t}}\,
\|f^*_t - f_{\tau_c}^M\|_{\mathcal{H}}^2
\right).
\end{align*}

From Proposition~\ref{OPbound6}, for  
\(\tilde{t} = C_{\kappa,p}^{-1} n \log^{-1}(1/\delta)\)  
we have
\(
\|\Sigma_{\tilde{t}}^{1/2} \widehat{\Sigma}_{\tilde{t}}^{-1/2}\| \le 2.
\)
Thus,

\begin{align*}
\|\mathcal{S}_M f^*_t - \mathcal{S}_M f_{\tau_c}^M\|^2_{L^2(\rho_x)}
\leq
2\left(
\|\widehat{S}_M(f^*_t - f_{\tau_c}^M)\|_n^2
+
\frac{1}{\tilde{t}}\,
\|f^*_t - f_{\tau_c}^M\|_{\mathcal{H}}^2
\right).
\end{align*}

The second term is controlled by Proposition~\ref{ftaubound} together with \ref{ineq5}:

\[
\frac{1}{\tilde{t}}\, \|f^*_t - f_{\tau_c}^M\|_{\mathcal{H}}^2
\leq
C_\bullet \log(1/\delta)\, \frac{1}{n}.
\]

It remains to bound
\(
\|\widehat{S}_M(f^*_t - f_{\tau_c}^M)\|_n^2.
\)
We decompose:

\begin{align*}
\|\widehat{S}_M(f_{\tau_c}^M - f^*_t)\|_n^2
\leq
2\|\widehat{S}_M f_{\tau_c}^M - \widehat{\mathcal{S}} g^*\|_n^2
+
2\|\widehat{\mathcal{S}}g^* - \widehat{S}_M f_t^*\|_n^2.
\end{align*}

The first term is bounded in Theorem~\ref{empnorm} by

\[
C_\bullet\, n^{-\frac{2r}{2r+b}}\, \log^2(1/\delta).
\]

For the second term, Proposition~\ref{concentrationineq2} gives

\begin{align*}
\|\widehat{\mathcal{S}}g^* - \widehat{S}_M f_t^*\|_n^2
\leq
\frac{
C_\bullet\, \|\mathcal{S}g^* - \mathcal{S}_M f_t^*\|_{L^2(\rho_x)}\, \log^2(1/\delta)
}{\sqrt{n}}
+
\|\mathcal{S} g^* - \mathcal{S}_M f^*_t\|^2_{L^2(\rho_x)}.
\end{align*}

Using Proposition~\ref{rfbiasbound} as before, we obtain

\begin{align*}
&\frac{
C_\bullet\|\mathcal{S}g^* - \mathcal{S}_M f_t^*\|_{L^2(\rho_x)}\log(1/\delta)
}{\sqrt{n}}
+
\|\mathcal{S} g^* - \mathcal{S}_M f^*_t\|^2_{L^2(\rho_x)}
\\[4pt]
&\leq
\frac{C_\bullet\, t^{-r}\log^2(1/\delta)}{\sqrt{n}}
+
C_\bullet\, t^{-2r}\log^2(1/\delta)
\\[4pt]
&\leq
C_\bullet\, n^{-\frac{2r}{2r+b}}\, \log^3(1/\delta),
\end{align*}

where the final step follows as in \eqref{dsjksjdksd}.  
This completes the proof.

\end{proof}

\subsection*{Upper Bound on  $\tau_c$}
\label{upperbound}

In this section, we provide an upper bound on the stopping rule. 
We recall the RFA stopping time. 
\begin{algorithm}[H]
\textbf{Input:}\vspace*{-0.9cm}\\

\hspace*{-4.5cm}
\begin{minipage}{1.2\linewidth}
\begin{align*}
\mathcal{M} &:= \{\lfloor n^{\frac{j}{\log n}} \rfloor \mid j = 1, \dots, \log n \}, 
 \text{ constants } \tilde{c}, c > 0, \\
 \mathcal{T}_t &:= \{ t e^p : p = 0, \dots, \lfloor \log(n/t) \rfloor \},\\ 
 L_T^m &:= \phi_T^{1/2}(K_m) K_m^{1/2},\\
V_{T,m} &:= \sigma^2 \max\{\widehat{\mathcal{N}}_{m}^{\phi}(T), \log\log n\}, \\
\,\widehat{\mathcal{N}}^{\phi}_m(T) &:= \operatorname{tr}(\phi_T(K_m) K_m).
\end{align*}
\end{minipage}
\vspace*{0.4cm}\\
 \textbf{For} $m \in \mathcal{M}$ \textbf{do:}\vspace*{-0.3cm}
\begin{align}
& \textbf{Set } \tau_{m,c} := \min\left\{t \in \mathbb{N} \,\Big|\,  
   \| L_T^m(Y - \widehat{\mathcal{S}}_m f_t^m) \|_n^2
   \le \tfrac{c V_{T,m}}{n},\quad \forall T \in \mathcal{T}_t \right\}\wedge m,\\[3pt]
& \textbf{Break if } \nonumber
  \Bigl(\tau_{m,c} < m \text{ and } 
  \tfrac{\tilde{c}\log n}{m} \le  
  \tfrac{\widehat{\mathcal{N}}_m^{\phi}(\tau_{m,c})}{n}\Bigr),\\[7pt]
& \textbf{Set } M := m \text{ and } \tau_c := \tau_{M,c}.\nonumber
\end{align}
\textbf{End for.}
\end{algorithm}

We bound $\tau_c$ from above by a theoretical optimal choice $\bar{T}_M$, where we set for all $m \in \mathcal{M}$,
\begin{align}\label{bartm}
\bar{T}_m := \min_{t \in \mathbb{N}}
\left\{
\| r_t(K_m) \widehat{\mathcal{S}} g^* \|_n^2
\leq
\frac{\widehat{\mathcal{N}}_m^{\phi}(t)}{n}
\right\}\wedge n.
\end{align}

\begin{proposition}\label{upboundrf}
Suppose that Assumption~\ref{ass:subgaus} holds.  
Then, for every $c \geq 2(1 + \hat{c} + \log(1/\delta))$ with $\delta > 0$ and some absolute constant $\hat{c}>0$, we have
\[
\mathbf{P}_\epsilon(\tau_c \leq \bar{T}_M) \geq 1 - \delta .
\]
\end{proposition}

\begin{proof}

By definition, $\tau_c \leq n$, and hence the case $\bar{T}_M = n$
can be excluded. However, since $M$ depends on $\epsilon$, the
quantity $\bar{T}_M$ is itself noise-dependent. We therefore proceed
with the following intersection argument.


\begin{align*}
\mathbf{P}_\epsilon(\tau_c \leq \bar{T}_M)
\;\geq\;
\mathbf{P}_\epsilon\!\left(
\tau_{c,m} \leq \bar{T}_m,
\ \forall m \in \mathcal{M}
\right)
=
\mathbf{P}_\epsilon\!\left(
\tau_{c,m} \leq \bar{T}_m,
\ \forall m \in \widetilde{\mathcal{M}}
\right),
\end{align*}
where
\[
\widetilde{\mathcal{M}}
:=
\{ m \in \mathcal{M} : \bar{T}_m < n \}.
\]

Next, define
\[
\tilde{\epsilon}_{t,m} := L_t^m \epsilon,
\qquad
\tilde{g}_{t,m}:=L_t^m\widehat{\mathcal{S}}g^*,
\qquad
L_t^m := K_m^{1/2}\phi_t^{1/2}(K_m).
\]

For any $m \in \widetilde{\mathcal{M}}$, the definition of $\tau_{c,m}$
implies
\begin{align*}
\mathbf{P}_\epsilon(\tau_{c,m} \leq \bar{T}_m)
&\geq
\mathbf{P}_\epsilon\!\left(
\| r_{\bar{T}_m}(K_m) L_T^m Y \|_n^2
<
c \frac{V_{T,m}}{n},
\ \forall T \in \mathcal{T}_{\bar{T}_m}
\right)
\\
&\geq
\mathbf{P}_\epsilon\!\left(
2\| r_{\bar{T}_m}(K_m)\tilde{g}_{T,m} \|_n^2
+
2\| \tilde{\epsilon}_{T,m} \|_n^2
<
c \frac{V_{T,m}}{n},
\ \forall T \in \mathcal{T}_{\bar{T}_m}
\right)
\\
&\geq
\mathbf{P}_\epsilon\!\left(
2\| r_{\bar{T}_m}(K_m)\widehat{\mathcal{S}}g^* \|_n^2
+
2\| \tilde{\epsilon}_{T,m} \|_n^2
<
c \frac{V_{T,m}}{n},
\ \forall T \in \mathcal{T}_{\bar{T}_m}
\right),
\end{align*}
where the last step uses $\phi_t(x)x \leq 1$.

Using the definition of $\bar{T}_m$, we further obtain for every $m \in \widetilde{\mathcal{M}}$,
\begin{align*}
&\mathbf{P}_\epsilon\!\left(
2\| r_{\bar{T}_m}(K_m)\widehat{\mathcal{S}}g^* \|_n^2
+
2\| \tilde{\epsilon}_{T,m} \|_n^2
<
c \frac{V_{T,m}}{n},
\ \forall T \in \mathcal{T}_{\bar{T}_m}
\right)
\\&\geq
\mathbf{P}_\epsilon\!\left(
2\| \tilde{\epsilon}_{T,m} \|_n^2
<
(c-2)\frac{V_{T,m}}{n},
\ \forall T \in \mathcal{T}_{\bar{T}_m}
\right).
\end{align*}

Combining the bounds yields
\begin{align*}
\mathbf{P}_\epsilon(\tau_c \leq \bar{T}_M)
&\geq
\mathbf{P}_\epsilon(
\tau_{c,m} \leq \bar{T}_m,
\ \forall m \in \widetilde{\mathcal{M}}
)
\\
&\geq
\mathbf{P}_\epsilon\!\left(
2\| \tilde{\epsilon}_{T,m} \|_n^2
<
(c-2)\frac{V_{T,m}}{n},
\ \forall T \in \mathcal{T}_{\bar{T}_m},
\ \forall m \in \widetilde{\mathcal{M}}
\right).
\end{align*}

From Proposition~\ref{epsbound2}, we have for some absolute constant $\hat{c}>0$,
\[
\mathbf{P}_\epsilon\left(
2\|\tilde{\epsilon}_T\|_n^2
<
\frac{
\hat{c}\sigma^2\left(\widehat{\mathcal{N}}_M^\phi(T) + \log(\log(n)/\delta)\right)
}{n},
\quad
\forall\, T \in \mathcal{T}_{\bar{T}_M}, \,\,\,\forall m\in\widetilde{\mathcal{M}}
\right)
\geq 1 - \delta.
\]

Thus the claim holds whenever  
\[
c \geq 2(1 + \hat{c} + \log(1/\delta)).
\]
\end{proof}

\subsection*{Proving Optimal Rates for NNs}
\label{oprates nns}

In the following, we prove Theorem~\ref{NNstop} and begin by recalling the results from \cite{nguyen2023neurons}.

\begin{theorem}[\cite{nguyen2023neurons} Theorem B.5]
\label{prop:second-term}
Let $\delta \in (0,1]$ and $\alpha < 1/\kappa^2$.  
Suppose Assumptions \ref{ass:kernel} and \ref{ass:neurons} are satisfied.  
If $M \geq C_{\bullet}\log(n)\,T$, then with probability at least $1-\delta$,
\begin{align*}
  \forall\, t \in [T]: \qquad 
  \|g_{\theta_t} - \mathcal{S}_Mf_t^M\|_{L^2(\rho_x)}^2 
  \;\leq\;
  C_\bullet\,\frac{\log^2(n)}{M}\, \|\theta_t - \theta_0\|^6 \, .
\end{align*}
\end{theorem}

\begin{theorem}[\mycite{nguyen2023neurons} Corollary 3.8]
\label{cor:weights}
Suppose the Assumptions \ref{ass:kernel}, \ref{ass:source}, \ref{ass:dim}, \ref{ass:neurons}, hold.  
Let $T = C_\bullet n^{\frac{1}{2r + b}}$ with $r=s+1/2$.  
Then, with probability at least $1 - \delta$,
\begin{equation}
\label{eq:ball-1}
\forall\, t \in [T], 
\qquad 
\|\theta_t - \theta_0\|
\;\leq\;
C_\bullet\,\log(T),
\end{equation}
provided that
\begin{align*}
 M \;\geq\; C_\bullet \log^4(T)\, T^{2r}.
\end{align*}
\end{theorem}

\paragraph{Remark.}
The source condition in Theorem~\ref{cor:weights} of \mycite{nguyen2023neurons} is stated in a slightly different form but is in fact equivalent to our Assumption~\ref{ass:source}.  
More precisely, \mycite{nguyen2023neurons} assume that for the operator 
\(\mathcal{L} = \mathcal{S}\mathcal{S}^*: L^2(\mathcal{X},\rho_x) \to L^2(\mathcal{X},\rho_{x})\)  
there exists some \(f \in L^2(\mathcal{X},\rho_x)\) such that
\[
g^* = \mathcal{L}^r f,
\qquad r = s + \tfrac{1}{2}, \;\; s \ge 0 .
\]

This condition is equivalent to our source assumption. Indeed, by Mercer’s representation theorem \citep{steinwart2008support}, we obtain
\[
\operatorname{ran}(\mathcal{S}\Sigma^{s})
\;=\;
\operatorname{ran}\!\bigl(\mathcal{L}^{\,r}\bigr)
=
\Bigl\{
\textstyle \sum_{i=1}^{\infty} \mu_i^{\,r}\, a_i\, \varphi_i
\,\bigm|\,
(a_i)\in \ell^2
\Bigr\},
\]
where $\Sigma = \mathcal{S}^*\mathcal{S}$ and $(\mu_i,\varphi_i)$ denote the eigenvalues and eigenvectors of $\mathcal{L}$. Hence, both formulations impose exactly the same regularity requirement on the target function $g^*$.

\begin{proof}[Proof of Theorem \ref{NNstop}]
We first note that the additional $\log n$ factor introduced in Algorithm~\ref{algom2}
only enlarges the estimate of the number of random features~$M$.
In particular, the proof of Theorem~\ref{L2norm} remains valid under this modified condition.
Therefore, if $\tau_c$ satisfies the condition on $M$ stated in Theorem~\ref{prop:second-term},
it follows from Theorem~\ref{prop:second-term} and Theorem~\ref{L2norm} that 

\begin{align*}
\|\widehat{\mathcal{S}}g^* - g_{\theta_{\tau_c}}\|_{L^2(\rho_x)}^2
&\leq 
\|\widehat{\mathcal{S}}g^* - \widehat{\mathcal{S}}_M f_{\tau_c}^M\|_{L^2(\rho_x)}^2
\;+\;
\|\widehat{\mathcal{S}}_M f_{\tau_c}^M - g_{\theta_{\tau_c}}\|_{L^2(\rho_x)}^2
\\
&\leq
C_\bullet\, n^{-\frac{2r}{2r+b}} \log^4(1/\delta)
\;+\;
C_\bullet\,\frac{\log^2(n)}{M}\,\|\theta_{\tau_c} - \theta_0\|^6 .
\end{align*}

By the definition of $M$,
\begin{align*}
C_\bullet\,\frac{\log^2(n)}{M}\,\|\theta_{\tau_c} - \theta_0\|^6
\;\leq\;
C_\bullet\,\frac{\widehat{\mathcal{N}}^{\phi}_M(\tau_c)}{n \log^6(n)}\,\|\theta_{\tau_c} - \theta_0\|^6.
\end{align*}
Applying Proposition~\ref{upboundrf} and Proposition~\ref{prop:effecdim5},
\begin{align*}
C_\bullet\,\frac{\widehat{\mathcal{N}}^{\phi}_M(\tau_c)}{n \log^6(n)}\,\|\theta_{\tau_c} - \theta_0\|^6
\;\leq\;
C_\bullet\, n^{-\frac{2r}{2r+b}}\,\frac{\|\theta_{\tau_c} - \theta_0\|^6}{\log^6(n)}.
\end{align*}
Thus, it remains to verify the conditions of Theorem \ref{prop:second-term} and \ref{cor:weights}:
\begin{itemize}
\item[i)] $\tau_c \leq C_\bullet n^{\frac{1}{2r+b}}$,
\item[ii)] $M \geq C_\bullet \log^4(n)\, n^{\frac{2r}{2r+b}}$.
\end{itemize}

\textbf{i)}  
By Proposition~\ref{upboundrf} and Assumption~\ref{ass:dim2},
\begin{align*}
\frac{\tau_c^b}{n}
\;\leq\;
\frac{\bar{T}_M^b}{n}
\;\leq\;
C_\bullet \frac{\mathcal{N}(\bar{T}_M)}{n}
\;\leq\;
C_\bullet\, n^{-\frac{2r}{2r+b}},
\end{align*}
where the last inequality follows from Proposition~\ref{prop:effecdim5}.  
Hence,
\[
\tau_c \leq C_\bullet\, n^{\frac{1}{2r+b}}.
\]

\textbf{ii)}  
Similarly, by definition of $M$,
\begin{align*}
\frac{1}{M}
\;\leq\;
C_\bullet\,\frac{\widehat{\mathcal{N}}(\bar{T}_M)}{\log^8(n)\,n}
\;\leq\;
C_\bullet\,n^{-\frac{2r}{2r+b}}\,\log^{-8}(n),
\end{align*}
which yields the desired lower bound.  
This proves the claim.
\end{proof}

\section{Technical Inequalities}
\label{appendixB}

\subsection{Operator Inequalities}
\label{operatorineqs}

\begin{remark}
The operator bounds below are stated for a fixed value of $M>0$.
They also remain valid for the data-dependent choice of $M$ defined in
Algorithm~\ref{algom}.

Indeed, let $\mathcal{M}$ denote the set of possible values from which
Algorithm~\ref{algom} selects $M$, and assume that $|\mathcal{M}| \leq \log n$.
For each fixed $m \in \mathcal{M}$, let $E_m$ be the event on which the
corresponding bound holds with failure probability at most $\tilde\delta$,
that is,

\[
    \mathbb{P}(E_m^c) \leq \tilde\delta .
\]

Then, for the data-dependent choice $M \in \mathcal{M}$,

\[
    E_M^c \subseteq \bigcup_{m \in \mathcal{M}} E_m^c .
\]

Hence, by the union bound,

\[
    \mathbb{P}(E_M^c)
    \leq \sum_{m \in \mathcal{M}} \mathbb{P}(E_m^c)
    \leq |\mathcal{M}| \tilde\delta
    \leq \log n \, \tilde\delta .
\]

Choosing $\tilde\delta = \delta/\log n$ therefore yields

\[
    \mathbb{P}(E_M) \geq 1-\delta .
\]

Consequently, the following propositions continue to hold for the
data-dependent choice of $M$, provided that the failure probability
$\delta$ appearing in their assumptions is replaced by $\delta/\log n$.
For instance, a condition of the form

\[
    n \geq t \log(n/\delta),
    \qquad
    M \geq  t \log(n/\delta),
\]

is replaced by

\[
    n \geq 2 t \log\!\left(n/\delta\right),
    \qquad
    M \geq  2t \log\!\left(n/\delta\right).
\]

Under these modified conditions, the corresponding bound holds with
probability at least $1-\delta$.
\end{remark}

\begin{proposition}[\cite{nguyen2023random} Proposition B.15]
\label{OPbound6}
Under Assumptions \ref{ass:kernel}, \ref{ass:dim} we have for  any $n\geq C_{\bullet}\,t\log(n/\delta)$  and $M\geq C_{\bullet}\,t \log(n/\delta)$ that with probability at least $1-\delta$

\[
\left\|\widehat{\Sigma}_{M,t}^{-\frac{1}{2}}\Sigma_{M,t}^{\frac{1}{2}}\right\| \leq 2, \quad  \left\|\widehat{\Sigma}_{M,t}^{\frac{1}{2}}\Sigma_{M,t}^{-\frac{1}{2}}\right\| \leq 2 .
\]

\end{proposition}

\medskip

\begin{proposition}[\cite{nguyen2023random} Proposition B.17]
\label{OPbound8}
For any $s>0$, 

\[n\geq C_{\bullet}\,\log^3 (n/\delta)\max\{t^{2s} ,\,  \mathcal{N}(t)\,t \}\,, \quad M\geq C_{\bullet}\, \log(n/\delta)t\,,\]

we have with probability at least $1-\delta$,

\begin{align*}
\left\|\widehat{\Sigma}_{M,t}^{-s} \Sigma_{M,t}^{s}\right\|\leq 2.
\end{align*}

\end{proposition}

\medskip

\begin{proposition}
\label{OPbound7}
Suppose Assumption \ref{ass:kernel} holds, then for any  $r\geq 0.5$ and

\begin{align*}
 n\geq C_{\bullet}\,t \log(n/\delta)\,, \quad M\geq C_{\bullet}\,\log(n/\delta)\cdot t^{2r}\,,
\end{align*}

we have that with probability at least $1-\delta$,

\[
\left\|K_{M,t}^{-(r\vee1)}K_{t}^{r}\right\| \leq 3t^{(1-r)^+}.
\]

\end{proposition}

\begin{proof}
The proof follows the same steps as in \cite{nguyen2023random}, Proposition B.16, with the only differences being the replacement of $\mathcal{L}_M=\mathcal{S}_M\mathcal{S}_M^*$ by $K_M=\widehat{\mathcal{S}}_M\widehat{\mathcal{S}}_M^*$ and $\mathcal{L} =\mathcal{S} \mathcal{S} ^*$ by $K =\widehat{\mathcal{S}} \widehat{\mathcal{S}} ^*$, respectively.
\end{proof}


\subsection{Effective Dimension Bounds}

\begin{proposition}[ \cite{celisse:hal-02548917} Lemma 6]
\label{Lemmamartin0}
 Let $\phi$ be a regularizer satisfying \ref{LFL}. Then for each $t \geq 0, M \in \{\mathbb{N},\infty\}$

\[
c_\phi \hat{\mathcal{N}}_M(t) \leq \hat{\mathcal{N}}_M^\phi(t) \leq 2(E \vee 1) \hat{\mathcal{N}}_M(t)\,.
\]

\end{proposition}

\medskip

\begin{proposition}[\cite{nguyen2023random} Proposition B.18]
\label{prop:effecdim2}
For any fixed $M\geq C_{\bullet}\,t \log(1/\delta)$, we have with probability at least $1-2\delta$,

\begin{align*}
&a)\,\,\,\,\mathcal{N}_{M}(t)\leq  \left(1+2\log\frac{2}{\delta}\right)4\mathcal{N}(t),\\
&b)\,\,\,\,\mathcal{N}(t)\leq  \left(1+2\log\frac{2}{\delta}\right)4\mathcal{N}_{M}(t).
\end{align*}

\end{proposition}

\medskip

\begin{proposition}
\label{prop:effecdim3}
For any $M\in \{\mathbb{N}, \infty\}$ , $n\geq C_{\bullet}\,t \log(1/\delta)$, we have with probability at least $1-2\delta$,

\begin{align*}
&a)\,\,\,\hat{\mathcal{N}}_{M}(t)\leq  \left(1+2\log\frac{2}{\delta}\right)4\mathcal{N}_{M}(t),\\
&b)\,\,\,\mathcal{N}_{M}(t)\leq  \left(1+2\log\frac{2}{\delta}\right)4\hat{\mathcal{N}}_{M}(t).
\end{align*}

\end{proposition}

\begin{proof}

 The proof follows exactly the same steps as in Proposition \ref{prop:effecdim2}.

\end{proof}

\subsection{Specific Technical Inequalities}

\medskip

\begin{proposition}
\label{ftaubound}
Given the Assumptions \ref{ass:kernel}, \ref{ass:source}, \ref{ass:dim} we have for any 

\[n> C_{\bullet}\,\log^2(1/\delta)t^{b+1/2}\,,\quad M>C_{\bullet}\,t \log(1/\delta),\,
\]

that with probability at least $1-\delta$,

\begin{align*}
&\|f_{t}^M\|_{\mathcal{H}_M} \leq R \kappa^{2s}+1.
\end{align*}

\end{proposition}

\begin{proof}
By definition we have,

\begin{align*}
\|f_{t}^{M}\|_{\mathcal{H}_M} &= \|\widehat{\mathcal{S}}_M^*\phi_{t}(K_M)Y\|_{\mathcal{H}_M}\\
&\leq\|\widehat{\mathcal{S}}_M^*\phi_{t}(K_M)\widehat{\mathcal{S}}_M g^*\|_{\mathcal{H}_M}+\|\widehat{\mathcal{S}}_M^*\phi_{t}(K_M)\epsilon\|_{\mathcal{H}_M}.
\end{align*}

For the first norm it follows from \eqref{def.phi}, 

\[
\|\widehat{\mathcal{S}}_M^*\phi_{t}(K_M)\widehat{\mathcal{S}}_M g^*\|_{\mathcal{H}_M}\leq \|g^*\|_\mathcal{H} \leq R \kappa^{2s}\,.
\]

For the second norm we have from \eqref{def.phi},

\begin{align*}
\|\widehat{\mathcal{S}}_M^*\phi_{t}(K_M)\epsilon\|_{\mathcal{H}_M}\leq \|\phi_{t}^{1/2}(K_M)\|\|\widehat{\mathcal{S}}_M^*\phi_{t}^{1/2}(K_M)\epsilon\|_{\mathcal{H}_M}\leq \sqrt{Et} \|\phi_{t}^{1/2}(K_M)K_M^{1/2}\epsilon\|_n
\end{align*}

Therefore we obtain from Proposition \ref{epsbound}, with probability at least $1-\delta$,

\begin{align*}
\|\widehat{\mathcal{S}}_M^*\phi_{t}(K_M)\epsilon\|_{\mathcal{H}_M}\leq \sqrt{Et} \|\phi_{t}^{1/2}(K_M)K_M^{1/2}\epsilon\|_n \leq\sqrt{Et} \frac{2\sigma^2c^{-1}}{n} (\hat{\mathcal{N}}_M^\phi(t)+\log(1/\delta))
\end{align*}

From Proposition \ref{Lemmamartin0}, \ref{prop:effecdim2} and \ref{prop:effecdim3}
we have with probability at least $1-\delta$, that $\hat{\mathcal{N}}_M^\phi(t)\leq C_E \log^2(1/\delta)\mathcal{N} (t)$, for some $C_E>0$.
By Assumption \ref{ass:dim},  it follows $\hat{\mathcal{N}}_M^\phi(t)\leq C_E \log^2(1/\delta)C_bt^{b}$. Using this bound we obtain,

\begin{align*}
\|\widehat{\mathcal{S}}_M^*\phi_{t}(K_M)\epsilon\|_{\mathcal{H}_M}\leq  C_{\sigma^2,E,b} \log^2(1/\delta)\frac{t^{b+1/2}}{n},
\end{align*}

for some constant $C_{\sigma^2,E,b}>0$. By assumption, we have $C_{\sigma^2,E,b} \log^2(1/\delta)\frac{t^{b+1/2}}{n}\leq 1$. This completes the proof.

\end{proof}

\medskip

\begin{proposition}\label{resgbound}
Given the Assumptions \ref{ass:subgaus},\ref{ass:kernel}, \ref{ass:source}, \ref{ass:dim}, then 
with probability at least $1-\delta$,

\begin{align*}
\|r_{t}(K_M) \widehat{\mathcal{S}} {g}^*\|^2_n \leq C_{\bullet}\,\left( t^{-2(r\vee1)}n^{\frac{2(1-r)^+}{2r+b  }}+ \left(\frac{\log\left(\frac{n}{\delta}\right)}{M}\vee n ^{-\frac{2r}{2r+b}}\right)\right).
\end{align*}

\end{proposition}

\begin{proof}

We define $K_{M,\tilde{t}}:= (K_M+\tilde{t}^{-1}I)$ and use $g^*=\Sigma ^s h$ with $\|h\|_{\mathcal{H} }\leq R$,  $r=s+0.5$ to obtain, 

\begin{align*}
\|r_{t}(K_M) \widehat{\mathcal{S}} {g}^*\|^2_n &\leq R^2 \|r_{t}(K_M) \widehat{\mathcal{S}}  \Sigma^{s}\|^2 \\
&\leq R^2 \|r_{t}(K_M) K_{M,\tilde{t}}^{r\vee1}\|^2\|K_{M,\tilde{t}}^{-(r\vee1)}K_{\tilde{t}}^{r}\|^2\|K_{\tilde{t}}^{-r}\widehat{\mathcal{S}}  \Sigma^{s} \|^2 \\
&\leq R^2 \|r_{t}(K_M) K_{M,\tilde{t}}^{r\vee1}\|^2\|K_{M,\tilde{t}}^{-(r\vee1)}K_{\tilde{t}}^{r}\|^2\|\widehat{\Sigma}^{-s}_{\tilde{t}} \Sigma^{s}_{\tilde{t}}\|^2 ,
\end{align*}

for any $\tilde{t}>0$. 
By Assumption \eqref{c_r}, we have for the first norm 

\[
\|r_{t}(K_M) K_{M,\tilde{t}}^{r\vee1}\|\leq c_{r\vee1}t^{-(r\vee 1)}+\tilde{t}^{-(r\vee 1)}.
\]

By choosing 

\[
\tilde{t}:= C_{\bullet}\,\,((\log(n/\delta))^{-1/(2r)}\,
 M^{1/2r}) \wedge n^{\frac{1}{2r+b}}),
\] 

we obtain for the second norm from Proposition \ref{OPbound7} that 

\[ \left\|K_{M,\tilde{t}}^{-(r\vee1)}K_{\tilde{t}}^{r}\right\| \leq 3\tilde{t}^{(1-r)^+}.\]

For the last norm, observe that for $s=0$ it is trivially bounded by $1$.
Let now $s>0$. By the choice of $\tilde t$, the assumptions of
Proposition~\ref{OPbound8} are satisfied, for a constant $C_{\bullet}$
depending on $s$ and $b$, since

\[
    n^{\frac{1}{2r+b}}
    =
    o\!\left(
        \log^{-3}(n)
        \max\left\{
            n^{\frac{1}{2s}},
            n^{\frac{1}{1+b}}
        \right\}
    \right).
\]

Hence, Proposition~\ref{OPbound8} yields

\[
    \left\|
        \widehat{\Sigma}_{M,\tilde t}^{-s}
        \Sigma_{M,\tilde t}^{s}
    \right\|
    \leq 2 .
\]

Plugging these bounds in implies 

\begin{align*}
\|r_{t}(K_M) \widehat{\mathcal{S}} {g}^*\|^2_n \leq C_{\bullet}\,( t^{-2(r\vee1)}+\tilde{t}^{-2(r\vee1)})\cdot\tilde{t}^{2(1-r)^+},
\end{align*}

for some $C_{\bullet}>0$. Plugging in the value of $\tilde{t}$ implies

\begin{align*}
\|r_{t}(K_M) \widehat{\mathcal{S}} {g}^*\|^2_n \leq C_{\bullet}\,\left( t^{-2(r\vee1)}n^{\frac{2(1-r)^+}{2r+b  }}+ \frac{\log\left(\frac{n}{\delta}\right)}{M}\vee n ^{-\frac{2r}{2r+b}}\right),
\end{align*}

for some $C_{\bullet}>0$ .
\end{proof}

\medskip

\begin{proposition}
\label{prop:effecdim5}
Given the Assumptions \ref{ass:subgaus},\ref{ass:kernel}, \ref{ass:source},\ref{ass:dim},  we have for 

\begin{align}
\bar{T}_M:=\min_{t\in\mathbb{N}}\left\{\|r_{t}(K_M)\widehat{\mathcal{S}}g^*\|^2_n\leq\frac{\hat{\mathcal{N}}_M^{\phi}(t)}{n}\right\}\wedge n,
\end{align}

where $M$ is defined in \ref{algom}, that with probability at least $1-\delta$,

$$
\frac{\hat{\mathcal{N}}_M ^\phi(\bar{T}_M)}{n} \leq C_{\bullet}\, n^{-\frac{2r}{2r+b}}\log(1/\delta),
$$

for some $C_{\bullet}\,,\tilde{c}>0$ depending on $r,b,\kappa,s,E$.

\end{proposition}

\begin{proof}

We begin by recalling that Lemma~\ref{Lemmamartin0} yields

\[
c_\phi \, \hat{\mathcal N}_M(t)
\;\le\;
\hat{\mathcal N}_M^\phi(t)
\;\le\;
2 (E \vee 1)\, \hat{\mathcal N}_M(t),
\]

and moreover

\[
\hat{\mathcal N}_M(t)
\;\le\;
2\,\hat{\mathcal N}_M(t/2),
\qquad \text{for all } t>0.
\]

Combining these bounds, we obtain

\[
\frac{\hat{\mathcal N}_M^\phi(\bar T_M)}{n}
\;\le\;
C \,
\frac{\hat{\mathcal N}_M^\phi(\bar T_M/2)}{n},
\]

for some constant $C>0$ depending only on $c_\phi$ and $E$.

\medskip

Since $\hat{\mathcal N}_M^\phi(t)$ is monotone increasing in $t$, it holds for any $t > \bar T_M/2$ that

\begin{align}
\frac{\hat{\mathcal N}_M^\phi(\bar T_M/2)}{n}
\;\le\;
\frac{\hat{\mathcal N}_M^\phi(t)}{n}.
\end{align}

Similarly, since $\|r_t(K_M)\widehat{\mathcal S}g^*\|_n^2$ is monotone decreasing in $t$, the definition of $\bar T_M$ implies that for any $t < \bar T_M/2$,

\begin{align}
\frac{\hat{\mathcal N}_M^\phi(\bar T_M/2)}{n}
\;\le\;
\|r_t(K_M)\widehat{\mathcal S}g^*\|_n^2.
\end{align}

Combining both regimes yields

\begin{align}
\label{eq:oracle-balance}
\frac{\hat{\mathcal N}_M^\phi(\bar T_M/2)}{n}
&\le
\min_{t \in \mathbb N}
\left(
\frac{\hat{\mathcal N}_M^\phi(t)}{n}
+
\|r_t(K_M)\widehat{\mathcal S}g^*\|_n^2
\right) \\
&\le
\frac{\hat{\mathcal N}_M^\phi(\tilde T)}{n}
+
\|r_{\tilde T}(K_M)\widehat{\mathcal S}g^*\|_n^2.
\nonumber
\end{align}

where we choose
\[
\tilde T
:=
C_{\bullet}\,\, n^{\frac{1}{2r+b}}.\]

We conclude that

\[
\frac{\hat{\mathcal N}_M^\phi(\bar T_M)}{n}
\;\le\;
C_{\bullet}\,
\left(
\frac{\hat{\mathcal N}_M^\phi(\tilde T)}{n}
+
\|r_{\tilde T}(K_M)\widehat{\mathcal S}g^*\|_n^2
\right),
\]

for some constant $C_{\bullet}\,>0$.

For the first summand we use Proposition  \ref{prop:effecdim3} and  \ref{Lemmamartin0}  to obtain with probability at least $1-\delta$, 

$$
\frac{\hat{\mathcal{N}}_M ^\phi(\tilde{T})}{n}\leq  C_E \frac{\mathcal{N}(\tilde{T})}{n}\log(1/\delta).
$$

Using Assumption \ref{ass:dim} we further obtain

$$
\frac{\hat{\mathcal{N}} ^\phi(\tilde{T})}{n}\leq C_{\bullet}\, n^{-\frac{2r}{2r+b}}\log(1/\delta).
$$

For the second summand $\|r_{\tilde{T}}(K_M)\widehat{\mathcal{S}}g^*\|^2_n$ we use Proposition \ref{resgbound} to obtain

\begin{align*}
\|r_{\tilde{T}}(K_M) \widehat{\mathcal{S}} {g}^*\|^2_n &\leq C_{\bullet}\,\left( \tilde{T}^{-2(r\vee1)}n^{\frac{2(1-r)^+}{2r+b  }}+ \frac{\log\left(\frac{n}{\delta}\right)}{M}\vee n ^{-\frac{2r}{2r+b}}\right)\\
&\leq C_{\bullet}\,\left( n^{-\frac{2r}{2r+b  }}+ \frac{\log\left(\frac{n}{\delta}\right)}{M}\right).
\end{align*}

By definition of $M$ we therefore have for $\tilde{c}>2C_{\bullet}\,\log(1/\delta)$, 

\begin{align*}
\|r_{\tilde{T}}(K_M) \widehat{\mathcal{S}} {g}^*\|^2_n \leq C_{\bullet}\, n^{-\frac{2r}{2r+b  }}+ \frac{\widehat{\mathcal{N}}_M^{\phi}(\tau_{c})}{2n}\leq C_{\bullet}\,n^{-\frac{2r}{2r+b  }}+ \frac{\widehat{\mathcal{N}}_M^{\phi}(\bar{T}_M)}{2n},
\end{align*}

where the last inequality follows from Proposition \ref{upboundrf}.

To sum up we obtain

\begin{align*}
\frac{\widehat{\mathcal{N}}_M^{\phi}(\bar{T}_M)}{n} \leq C_{\bullet}\, n^{-\frac{2r}{2r+b  }}\log(1/\delta)+ \frac{\widehat{\mathcal{N}}_M^{\phi}(\bar{T}_M)}{2n},
\end{align*}

for some $C_{\bullet}\,>0$. This completes the claim.
\end{proof}

For the following propositions we define 

\[
f^*_t:=\mathcal{S}^*_M\phi_t (\mathcal{L}_M)\mathcal{S}g^* \in\mathcal{H}_M\,.
\]
\medskip

\begin{proposition}[\cite{nguyen2023random} Proposition A.11.]
\label{ineq5}
Given Assumption \ref{ass:source} we have with probability at least $1-\delta$,
\begin{align*}
\|f^*_t\|_{\mathcal{H}_M} &\leq 2 \kappa^{2r} R ,\\
\end{align*}
provided that $M\geq C_{\bullet}\, \log(n/\delta)t$.
\end{proposition}

\medskip

\medskip

\begin{proposition}[\cite{nguyen2023random} Proposition A.1.]
\label{rfbiasbound}
Given the Assumptions \ref{ass:subgaus}, \ref{ass:kernel}, \ref{ass:source},\ref{ass:dim}, then we have for some $n_0\geq0, C_{\bullet}\,>0$ and all $n\geq n_0$, $M\geq \log(n/\delta)\,t^{2r}$, with probability at least $1-\delta$, 
\begin{align*}
\|\mathcal{S}g^*-\mathcal{S}_Mf_t^*\|_{L^2(\rho_x)}^2 \leq C_{\bullet}\, \,t^{-2r}\log^2(1/\delta).
\end{align*}

\end{proposition}

\medskip
\medskip

\begin{proposition}[\cite{nguyen2023random} Proposition A.23.]
\label{concentrationineq2}
Given the Assumptions \ref{ass:kernel}, \ref{ass:source}, the following event holds with probability at least $1-\delta$, 

\begin{align*}
\left|\left\|\widehat{\mathcal{S}}g^*-\widehat{\mathcal{S}}_M f_T^*\right\|_n^2-\left\|\mathcal{S}g^*-\mathcal{S}_M f_T^*\right\|_{L^2(\rho_x)}^2\right| \leq  \frac{C_{\bullet}\, \left\|\mathcal{S}g^*-\mathcal{S}_M f_T^*\right\|_{L^2(\rho_x)}}{\sqrt{n}}\,\log \left(\frac{2}{\delta}\right).
\end{align*}

\end{proposition}
  
\medskip
\medskip


\subsection*{Concentration Inequalities}
\label{concineqss}

\begin{proposition}[\mycite{celisse:hal-02548917}(Lemma 28)]\label{epsbound}

Given the Assumption \ref{ass:subgaus} we have for every $t \geq 0, M>0$ and every $y>0$, 

$$
\begin{aligned}
\mathbf{P}_\epsilon\left(\left\|K_M^{1 / 2} \phi_t^{1 / 2}\left(K_M\right) \epsilon\right\|_n^2\geq\frac{\sigma^2}{n} \hat{\mathcal{N}}_M^\phi(t)+y\right) & \leq \exp \left(-c\left(\frac{n^2 y^2}{\sigma^4 \hat{\mathcal{N}}_M^\phi(t)} \wedge \frac{n y}{\sigma^2}\right)\right) .
\end{aligned}
$$

Setting $y=\frac{\sigma^2c^{-1}}{n} (\hat{\mathcal{N}}_M^\phi(t)+\log(1/\delta))$ gives,
$$
\begin{aligned}
\mathbf{P}_\epsilon\left(\left\|K_M^{1 / 2} \phi_t^{1 / 2}\left(K_M\right) \epsilon\right\|_n^2\geq\frac{2\sigma^2c^{-1}}{n} (\hat{\mathcal{N}}_M^\phi(t)+\log(1/\delta))\right)  \leq \delta .
\end{aligned}
$$
\end{proposition}

The constant $c>0$ originates from the Hanson-Wright inequality and  can be followed in the work of \mycite{rudelson2013hansonwrightinequalitysubgaussianconcentration} Theorem 1.

\begin{proposition}\label{epsbound2}
Given the Assumption \ref{ass:subgaus}, we have for any finite sets $\mathcal{A},\mathcal{B}\subset\mathbb{N},\,  |\mathcal{A}||\mathcal{B}|<\infty$, 

$$
\begin{aligned}
\mathbf{P}_\epsilon\left(\max_{t\in \mathcal{A}, m \in \mathcal{B}}\left\|\tilde\epsilon_{t,m}\right\|_n^2-\frac{2\sigma^2c^{-1}}{n} (\hat{\mathcal{N}}_m^\phi(t)+\log(|\mathcal{A}||\mathcal{B}|/\delta))>0\right) \leq \delta ,
\end{aligned}
$$
for $\tilde{\epsilon}_{t,m}:=K_m^{1 / 2} \phi_t^{1 / 2}\left(K_m\right) \epsilon$ and $c>0$ from Proposition \ref{epsbound}.
\end{proposition}

\begin{proof}
From Proposition \ref{epsbound} we have
\begin{align*}
&\mathbf{P}_\epsilon\left(\max_{t\in \mathcal{A}, m \in \mathcal{B}}\left\|\tilde\epsilon_{t,m}\right\|_n^2-\frac{2\sigma^2c^{-1}}{n} (\hat{\mathcal{N}}_m^\phi(t)+\log(|\mathcal{A}||\mathcal{B}|/\delta))>0\right)\\
&\leq\sum_{t\in \mathcal{A}, m \in \mathcal{B}}\mathbf{P}_\epsilon\left(\left\|\tilde\epsilon_{t,m}\right\|_n^2-\frac{2\sigma^2c^{-1}}{n} (\hat{\mathcal{N}}_m^\phi(t)+\log(|\mathcal{A}||\mathcal{B}|/\delta))>0\right)\leq \delta.
\end{align*}

\end{proof}

\begin{proposition}
\label{prop:lownoise}
Let $\epsilon\sim \mathcal N(0,\sigma^2 I_n)$ and define
\[
Z \;:=\; \|\epsilon\|_n^2.
\]
Then we have
\[
\mathbb P\!\left(\,|Z-\sigma^2|\;\ge\; \,\frac{\sigma^2}{\sqrt n}\right)\;\ge\; \frac{1}{60}.
\]

\end{proposition}

\begin{proof}
Write $\epsilon_i=\sigma G_i$ with $G_i\sim \mathcal N(0,1)$ i.i.d., and set
\[
\xi \;:=\; \sum_{i=1}^n G_i^2 \;\sim\; \chi^2_n.
\]
Then
\[
Z \;=\; \frac{\sigma^2}{n}\xi,
\qquad\text{hence}\qquad
Z-\sigma^2 \;=\; \frac{\sigma^2}{n}(\xi-n).
\]
Let $X := (Z-\sigma^2)^2\ge 0$. By the Paley--Zygmund inequality (see for example \mycite{10.1093/acprof:oso/9780199535255.001.0001}), for any $\theta\in(0,1)$,
\[
\mathbb P\!\left(X \ge \theta\,\mathbb E X\right)
\;\ge\;
(1-\theta)^2\,\frac{(\mathbb E X)^2}{\mathbb E[X^2]}.
\]
Using the central moments of a chi-square random variable,
\[
\mathbb E[(\xi-n)^2] = 2n,
\qquad
\mathbb E[(\xi-n)^4] = 12n^2+48n,
\]
we obtain
\[
\mathbb E X
= \mathbb E[(Z-\sigma^2)^2]
= \frac{\sigma^4}{n^2}\,\mathbb E[(\xi-n)^2]
= \frac{2\sigma^4}{n},
\]
and
\[
\mathbb E[X^2]
= \mathbb E[(Z-\sigma^2)^4]
= \frac{\sigma^8}{n^4}\,\mathbb E[(\xi-n)^4]
= \frac{\sigma^8}{n^4}\,(12n^2+48n)
= \frac{12\sigma^8}{n^2}\Bigl(1+\frac{4}{n}\Bigr).
\]
Choosing $\theta=1/2$ yields
\[
\mathbb P\!\left(X \ge \tfrac12\,\mathbb E X\right)
\;\ge\;
\frac{1}{4}\cdot
\frac{\bigl(\frac{2\sigma^4}{n}\bigr)^2}{\frac{12\sigma^8}{n^2}(1+\frac{4}{n})}
=
\frac{1}{12\,(1+\frac{4}{n})}
\;\ge\;
\frac{1}{60},
\]
where the last inequality uses $1+\frac{4}{n}\le 5$ for all $n\ge 1$. Finally,
\[
X \ge \tfrac12\,\mathbb E X
\quad\Longleftrightarrow\quad
|Z-\sigma^2|
\ge 
\frac{\sigma^2}{\sqrt n}.
\]
This proves the claim.
\end{proof}

\end{document}